\pdfoutput=1

\documentclass[11pt]{article}

\PassOptionsToPackage{dvipsnames,usename}{xcolor}
\usepackage[preprint]{acl}
\usepackage{times}
\usepackage{latexsym}

\usepackage[T1]{fontenc}
\usepackage[utf8]{inputenc}

\usepackage{microtype}

\usepackage{inconsolata}

\usepackage{graphicx}

\usepackage{xspace}
\usepackage{amsfonts}
\usepackage{amsthm}
\usepackage{amsmath}
\usepackage{amssymb}
\usepackage{bbm}
\usepackage{caption}
\usepackage{subcaption}
\usepackage{xcolor}
\usepackage{scalerel}
\usepackage{booktabs}
\usepackage{multirow}
\usepackage{xfrac} 
\usepackage{nicefrac} 
\usepackage{thmtools} 
\usepackage{thm-restate}
\usepackage[most]{tcolorbox}
\usepackage{multicol}

\usepackage{pifont}%
\usepackage{xcolor}
\usepackage[breaklinks]{hyperref}
\definecolor{darkblue}{rgb}{0, 0, 0.5}
\hypersetup{colorlinks=true, citecolor=darkblue, linkcolor=darkblue, urlcolor=darkblue}

\newtheorem{defin}{Definition}
\newtheorem{lemma}{Lemma}

\usepackage{cleveref}
\crefname{section}{\S}{\S\S}
\Crefname{section}{\S}{\S\S}
\crefname{table}{Tab.}{}
\crefname{figure}{Fig.}{Figs.}
\crefname{algorithm}{Alg.}{}
\crefname{equation}{eq.}{eqs.}
\crefname{appendix}{App.}{}
\crefname{theorem}{Theorem}{}
\crefname{lemma}{Lemma}{Lemmas}
\crefname{prop}{Prop.}{}
\crefname{defin}{Defn.}{}
\crefname{cor}{Corollary}{}
\crefname{observation}{Observation}{}
\crefname{assumption}{Assumption}{}
\crefformat{section}{\S#2#1#3}
\crefformat{footnote}{#2\footnotemark[#1]#3}

\usepackage{tabularray}
\usepackage{cellspace}
\newcommand\cincludegraphics[2][]{\raisebox{-0.3\height}{\includegraphics[#1]{#2}}}
\usepackage{setspace}
\makeatletter
\DeclareRobustCommand*{\escapeus}[1]{%
    \begingroup\@activeus\scantokens{#1\endinput}\endgroup}
\begingroup\lccode`\~=`\_\relax
\lowercase{\endgroup\def\@activeus{\catcode`\_=\active \let~\_}}
\makeatother

\usepackage[scaled=0.85]{helvet}
\newcommand{\makesf}[1]{\textsf{{\escapeus{#1}}}}

\DeclareMathOperator*{\expect}{\mathbb{E}}

\definecolor{mygray}{rgb}{.3, .3, .3}
\definecolor{mygreen}{rgb}{0, .5, 0}
\definecolor{myred}{rgb}{.6, 0.15, 0.15}
\newcommand{\charactercolor}{ForestGreen}
\newcommand{\wordcolor}{orange}
\newcommand{\subwordcolor}{purple}

\newcommand{\mycharacter}[2]{\newcommand{#1}{{\color{\charactercolor} #2}}}
\newcommand{\myword}[2]{\newcommand{#1}{{\color{\wordcolor} #2}}}
\newcommand{\mysubword}[2]{\newcommand{#1}{{\color{\subwordcolor} #2}}}

\mycharacter{\character}{c}
\myword{\word}{w}
\mysubword{\subword}{s}

\mycharacter{\characterspace}{\mathcal{C}}
\myword{\wordspace}{\mathcal{W}}
\mysubword{\subwordspace}{\mathcal{S}}
\mycharacter{\alphabet}{\characterspace}
\myword{\lexicon}{\wordspace}
\mysubword{\vocab}{\subwordspace}
\mysubword{\vocabeos}{\overline{\subwordspace}}

\mycharacter{\characters}{\mathbf{\character}}
\myword{\words}{\mathbf{\word}}
\mysubword{\subwords}{\mathbf{\subword}}

\mycharacter{\Character}{C}
\myword{\Word}{W}
\mysubword{\Subword}{S}
\mycharacter{\Characters}{\mathbf{\Character}}
\myword{\Words}{\mathbf{\Word}}
\mysubword{\Subwords}{\mathbf{\Subword}}

\newcommand{\charmap}{\mathbb{S}}
\newcommand{\charmapfunc}[1]{\genfrac{}{}{0pt}{}{\raisebox{-0.7ex}{\scaleto{\charmap}{6pt}}}{\scaleto{#1}{4pt}}}

\newcommand{\wordsubwordmap}{\charmapfunc{\lexicon\to\vocab^{*}}}

\newcommand{\subwordcharsmap}{\charmapfunc{\vocab^{*}\to\alphabet^{*}}}
\newcommand{\charsubwordsmap}{\charmapfunc{\alphabet^{*}\to\vocab^{*}}}
\newcommand{\wordcharsmap}{\charmapfunc{\lexicon^{*}\to\alphabet^{*}}}
\newcommand{\charwordsmap}{\charmapfunc{\alphabet^{*}\to\lexicon^{*}}}
\newcommand{\wordsubwordsmap}{\charmapfunc{\lexicon^{*}\to\vocab^{*}}}
\newcommand{\subwordwordsmap}{\charmapfunc{\vocab^{*}\to\lexicon^{*}}}

\newcommand{\charmapfuncmath}[1]{\underset{\scaleto{#1}{4pt}}{\scaleto{\charmap}{6pt}}}

\newcommand{\wordsubwordmapmath}{\charmapfuncmath{\lexicon\to\vocab^{*}}}

\newcommand{\subwordcharsmapmath}{\charmapfuncmath{\vocab^{*}\to\alphabet^{*}}}
\newcommand{\charsubwordsmapmath}{\charmapfuncmath{\alphabet^{*}\to\vocab^{*}}}
\newcommand{\wordcharsmapmath}{\charmapfuncmath{\lexicon^{*}\to\alphabet^{*}}}
\newcommand{\charwordsmapmath}{\charmapfuncmath{\alphabet^{*}\to\lexicon^{*}}}
\newcommand{\wordsubwordsmapmath}{\charmapfuncmath{\lexicon^{*}\to\vocab^{*}}}
\newcommand{\subwordwordsmapmath}{\charmapfuncmath{\vocab^{*}\to\lexicon^{*}}}

\mysubword{\subseqs}{\mathcal{S}}

\newcommand{\vtheta}{\boldsymbol{\theta}}
\newcommand{\ptheta}{p_{\scaleto{\vtheta}{4pt}}}
\newcommand{\R}{\mathbb{R}}
\newcommand{\one}{\mathbbm{1}}
\newcommand{\eos}{\texttt{eos}\xspace}
\newcommand{\defeq}{\mathrel{\stackrel{\textnormal{\tiny def}}{=}}}
\newcommand\sbullet[1][.5]{\mathbin{\vcenter{\hbox{\scalebox{#1}{$\bullet$}}}}}

\newcommand{\pprefix}{\mathbb{P}}
\newcommand{\powerset}{\mathcal{P}}
\newcommand{\elementsetfunc}[1]{\Psi_{\scaleto{#1}{4pt}}}
\newcommand{\wordsset}{\elementsetfunc{\lexicon}}
\newcommand{\subwordsset}{\elementsetfunc{\vocab}}

\newcommand{\pprefixfunc}[1]{\pprefix_{\scaleto{#1}{4pt}}}
\newcommand{\pprefixwords}{\pprefixfunc{\lexicon}}
\newcommand{\pprefixsubwords}{\pprefixfunc{\vocab}}

\newcommand{\bow}{\texttt{bow}\xspace}
\newcommand{\eow}{\texttt{eow}\xspace}
\newcommand{\midstring}{\texttt{mid}\xspace}

\newcommand{\subwordschosen}{\subwords^{\words}}
\newcommand{\subwordschosenprefix}{\subwords^{\words_{<t}}}
\newcommand{\subwordschosenfull}{\subwords^{\words_{<t} \circ \word}}
\newcommand{\subwordschosenword}{\subwords^{\word}}

\newcommand{\characterspacebow}{\characterspace_{\texttt{bow}}}
\newcommand{\characterspaceboweos}{\overline{\characterspace}_{\texttt{bow}}}
\newcommand{\vocabbow}{\vocab_{\texttt{bow}}}
\newcommand{\vocabboweos}{\vocabeos_{\texttt{bow}}}
\newcommand{\vocabboweospluscont}{\overline{\vocab_{\texttt{bow}}\circ\vocab^*}}
\newcommand{\vocabeow}{\vocab_{\texttt{eow}}}
\newcommand{\vocabmid}{\vocab_{\texttt{mid}}}
\newcommand{\vocabmideos}{\vocabeos_{\midstring}}
\newcommand{\vocabmideospluscont}{\overline{\vocab_{\midstring}\circ\vocab^*}}

\newcommand{\textwords}[1]{\texttt{\textcolor{\wordcolor}{#1}}}
\newcommand{\textchar}[1]{\textit{\textcolor{\charactercolor}{#1}}}
\newcommand{\texttokenised}[1]{\textit{\textcolor{\subwordcolor}{#1}}}
\newcommand{\wordsubwordsmaplargeequation}{\!\!\!\!\!\wordsubwordsmap\!\!}
\newcommand{\circlargeequation}{\!\!\!\!\circ\!\!}

\newcommand{\emptystring}{\texttt{``''}}

\newcommand{\deltallh}{\Delta_{\mathrm{llh}}}
\newcommand{\regressorone}{f_{\psi_1}}
\newcommand{\regressor}{f_{\psi}}
\newcommand{\regressortwo}{f_{\psi_2}}
\newcommand{\dataset}{\mathcal{D}}
\newcommand{\vx}{\mathbf{x}}

\newcommand{\expectsurp}{\expect[\surp(\word_t)]}

\newcommand{\expectsurpsquared}{\frac{\scaleto{\expect[\surp^2(\word_t)]}{8pt}}{\scaleto{\expect[\surp(\word_t)]}{8pt}}}

\newcommand*{\circled}[1]{\tikz[baseline=(char.base)]{
        \node[shape=circle,draw,inner sep=1pt] (char) {\normalfont{\small #1}};}}

\newcounter{bugfixcounter}
\DeclareRobustCommand{\newbugfixcounter}[1]{\refstepcounter{bugfixcounter}\textbf{Bug Fix~\circled{\thebugfixcounter}\label{#1}}\xspace}
\DeclareRobustCommand{\usebugfixcounter}[1]{\textbf{Bug Fix~\circled{\ref{#1}}}\xspace}

\newcommand{\defn}[1]{\textbf{#1}}
\newcommand{\wordexample}[1]{\textit{#1}}

\newcommand{\stringequiv}{\mathrel{\stackrel{{\scaleto{\Delta}{3pt}}}{=}}}
\newcommand{\stringequivtosubwords}{\mathrel{\stackrel{{\scaleto{\lexicon^{*} \to \vocab^{*}}{3pt}}}{\implies}}}
\newcommand{\stringequivtowords}{\mathrel{\stackrel{{\scaleto{\vocab^{*} \to \lexicon^{*}}{3pt}}}{\implies}}}

\title{How to Compute the Probability of a Word}

\newcommand{\ethid}{{\includegraphics[scale=0.045]{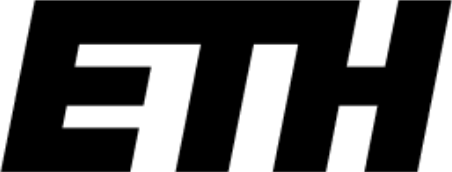}}}
\newcommand{\ethemailadress}[1]{\href{mailto:#1@inf.ethz.ch}{\texttt{#1}}}

\author{Tiago Pimentel,
Clara Meister
\\
  \ethid \\
  \{\ethemailadress{tiago\!.\!pimentel}, 
  \ethemailadress{clara\!.\!meister}%
  \}@\texttt{inf\!.\!ethz\!.\!ch}
}

\begin{document}
\maketitle
\begin{abstract}
Language models (LMs) estimate a probability distribution over strings in a natural language;
these distributions are crucial for computing perplexity and surprisal in linguistics research.
While we are usually concerned with measuring these values for words, most LMs operate over subwords.
Despite seemingly straightforward, accurately computing probabilities over one unit given probabilities over the other requires care. 
Indeed, we show here that many recent linguistic studies have been incorrectly computing these values. 
This paper derives the correct methods for computing word probabilities, highlighting issues when relying on language models that use beginning-of-word (\bow)-marking tokenisers, e.g., the GPT family. 
Empirically, we show that correcting the widespread bug in probability computations affects measured outcomes in sentence comprehension and lexical optimisation analyses. 

\begin{tblr}{colspec = {Q[c,m]|X[l,m]}, stretch = 0}
    \cincludegraphics[width=1.2em, keepaspectratio]{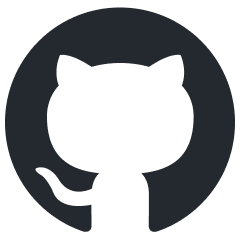}
     & \setstretch{.5}\href{https://github.com/tpimentelms/probability-of-a-word}{{\makesf{tpimentelms/probability-of-a-word}}} \\
\end{tblr}
\noindent
\begin{tblr}{colspec = {Q[c,m]|X[l,m]}, stretch = 0}
    \cincludegraphics[width=1.2em, keepaspectratio]{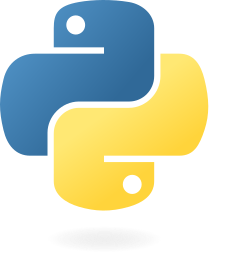}
     & \setstretch{.5}{\makesf{pip install wordsprobability}} \\
\end{tblr}
\end{abstract}

\section{Introduction}

Language models (LMs) define probability distributions.
After being trained on language data, these models can be used to compute estimates of the probability of a sequence of characters $\characters \in \characterspace^*$, or of a word $\word_t \in \wordspace$ in context $\words_{<t} \in \wordspace^*$. 
While deriving such estimates is now rarely the explicit goal of training such models,\footnote{Rather, LMs have become known for their high performance on downstream natural language processing (NLP) tasks \cite{gpt2,Touvron2023LLaMAOA}.} this use case is still critical in several fields. 
Estimating the probability of a sequence of characters, for instance, is necessary to compute a model's perplexity; a core evaluation metric in LM training.
Estimating the probability of a word in context is necessary to compute a \defn{word's surprisal}:  $-\log p(\word_t \mid \words_{<t})$, an important value in both psycho- and computational linguistics \citep{hale2001probabilistic,levy2007speakers,piantadosi2011word,pimentel-etal-2023-revisiting}.

\newcommand{\mathcomment}[1]{\text{\textcolor{gray}{#1}}}

\begin{figure}[t]
    \centering
\begin{tcolorbox}[colback=cyan!5!white,colframe=cyan!75!black,left=4pt,title=\!\!TL;DR: {\small How to correctly compute word probabilities}]
    \vspace{-5pt}
    {\small \textcolor{gray}{ Given a word $\word$ in context $\words_{<t}$, let $\subwordschosenword$ and $\subwordschosenprefix$ be their respective subword sequences output by a tokeniser. Further, let:}
    \vspace{-8pt}
    \begin{align}
        &\quad\, p\left(\subwordschosenword \mid \subwordschosenprefix\right) = \prod_{i=1}^{|\subwordschosenword|}p\left(\subword_i^{\word} \mid \subwordschosenprefix \circ \subwords_{<i}^{\word}\right)\nonumber 
        \\[5pt]
        &\sbullet[.75]\mathcomment{LM with end-of-word marking tokeniser} \nonumber \\
        &\quad\, p(\word \mid \words_{<t}) 
        = p\left(\subwordschosenword \mid \subwordschosenprefix\right)  \nonumber \\
        &\sbullet[.75]\mathcomment{LM with beginning-of-word marking tokeniser} \nonumber \\
        &\quad\, p(\word \mid \words_{<t}) 
        = p\left(\subwordschosenword \mid \subwordschosenprefix\right)
        \,
        \underbrace{\frac{\!\!\!\sum\limits_{\{\subword \in \vocabboweos\}}\!\!\!\!\! p\left(\subword \mid \subwordschosenprefix\circ\subwordschosenword\right)}
        {\!\!\!\sum\limits_{\{\subword \in \vocabboweos\}}\!\!\!\!\! p\left(\subword \mid \subwordschosenprefix\right)}}_{\text{``bug'' fix}} \nonumber 
    \end{align}}
\end{tcolorbox}
    \vspace{-10pt}
    \caption{
    Equations for computing a word's contextual probability $p(\word \mid \words_{<t})$ using a subword-based LM $p\left(\subword_t \mid \subwords_{<t}\right)$. 
    $\vocabboweos$ is a subset of the tokeniser's vocabulary marking beginnings of words.
    The ``bug'' fix can be computed for ``free'', i.e., within a single model pass.
    }
    \label{fig:tldr}
    \vspace{-7pt}
\end{figure}

Notably, most recent LMs operate over \defn{subwords} \cite{sennrich-etal-2016-neural,kudo-richardson-2018-sentencepiece}: sequences of characters that frequently occur together. This is done for both optimisation and efficiency reasons \cite{galle-2019-investigating,Mielke2021BetweenWA,zouhar-etal-2023-tokenization}. 
Subwords, however, do not necessarily constitute actual words, as defined by a language's lexicon.\footnote{Despite the name, which we use out of convention, a subword need not strictly be a subunit of a word. For example, subwords can span multiple words, containing the markers used to delineate words, e.g., white spaces.} 
At least superficially, converting from a probability distribution over subwords $p(\subwords)$ into one over characters $p(\characters)$ or words  
$p(\words)$ appears straightforward. 
However, some technical details are easy to overlook.
For example, several sequences of subwords $\subwords$ can map to a single sequence of characters $\characters$, implying an accurate computation of $p(\characters)$ should marginalise over these options \cite{cao-rimell-2021-evaluate}.

In this work, we discuss how to correctly compute a word's contextual probability: $p(\word_t \mathop{\mid} \words_{<t})$.
This value's computation depends on the choice of tokeniser used to define an LM's vocabulary.
When using an end-of-word (\eow)-marking tokeniser, computing $p(\word_t \!\!\mid\! \words_{<t})$ is simple.
However, when using a beginning-of-word (\bow)-marking tokeniser, correctly computing this value is not as straightforward. 
We derive methods for these tokenisation schemes, which we present in \cref{fig:tldr}.
Since many widely-used LMs employ \bow-marking tokenisers (e.g., the GPT models, Pythia, Mistral), this highlights a wide-spread ``bug'' in how most recent psycholinguistics and computational linguistics works compute word probabilities \citep[present in, e.g.,][]{oh2023why,wilcox-etal-2023-language,pimentel-etal-2023-revisiting,shain2024large}.\footnote{Concurrent work by \citet{oh2024leading} points out this same issue and proposes a solution similar to \usebugfixcounter{bugfix:bow} (in our \Cref{thm:probofbowtokeniser}). 
}

Empirically, we evaluate how correcting this computation affects the results of two prior empirical analyses: one on sentence comprehension and another on the lexicon's communicative efficiency.
While these studies' conclusions do not change, we do observe statistically significant differences between the measured quantities when using the correct vs.\ buggy methods for computing word probabilities.
We conclude this methodological choice may impact empirical analyses, and that future work should adopt these proposed corrections.\looseness=-1

\vspace{-3pt}
\section{What is a Word?}
\vspace{-2pt}

Despite decades of discussion and debate, there is no single, widely accepted definition of what constitutes a word  \citep{haspelmath2023defining}. 
Typically, definitions are made with respect to some system within the language, such as its orthography, phonology, or grammar.
As a concrete example, one can delineate words using the sound system of a language: if we assume words define  
the domain over which certain phonological processes operate (e.g., vowel harmony), we can delineate words based on those processes' boundaries  \citep{Hall1999StudiesOT,nespor2007prosodic}.
Alternatively, one could define words as grammatical elements (e.g., a root plus affixes) that are cohesive, occur in a fixed order, and have a coherent meaning \citep{Dixon2003WordWA}.
Notably grammatical and phonological words are non-isomorphic.
For example, English hyphenated elements like \emph{editor-in-chief} or \emph{mother-in-law} are typically analysed as a single grammatical word that contains multiple phonological words \citep{Dixon2003WordWA}.

We abstain from this broader discussion here. 
While we use the definition common to natural language processing applications---where words are defined orthographically\footnote{Orthographic words are defined as sequences of characters surrounded by white space or other special delimiters. One such delimiter is \wordexample{'}, present in the English clitic \wordexample{'s}.}---our methods only assume the existence of a deterministic set of rules for segmenting a string of characters into words.

\vspace{-3pt}
\section{Words and Distributions Over Them}
\vspace{-2pt}

Let $\lexicon$ be a lexicon---the (potentially infinite) set of all words in a language---and $\word \in \lexicon$ a word in this lexicon. 
Further, let $\words \in \lexicon^*$ be a sequence of words; $\lexicon^*$ denotes the set of all finite-length word sequences. 
Now, assume distribution $p$ describes the probability with which users of this language produce sequences $\words$.
We can decompose these probabilities autoregressively as:
\begin{align}
    p(\words) = p(\eos\mid \words)\prod_{t=1}^{|\words|} p(\word_t \mid \words_{<t})
\end{align}
where \eos is a special end-of-sequence symbol that makes
this probability distribution over $\lexicon^*$ valid.\footnote{See \citet{du-etal-2023-measure} for a longer discussion on when probability distributions over $\lexicon^*$ are valid.}

This paper is concerned with the proper method for computing the probability of a word in context, i.e., $p(\word_t \mathop{\mid} \words_{<t})$, using a pretrained language model.
To this end, we first discuss its equivalence to other quantities, which will ultimately reveal a flaw in prior approaches to its computation.
We start by defining a probability function $\pprefixwords$, which operates over sets of strings $\wordsset \subseteq \lexicon^*$.

\begin{defin} \label{defn:prob_words}
     Given distribution $p(\words)$, we define the probability function $\pprefixwords: \powerset(\lexicon^{*}) \to [0,1]$, which returns the probability of occurrence of any word sequence $\words \in \wordsset \subseteq \lexicon^*$. 
     As these events are disjoint,  $ \pprefixwords(\wordsset)$ can be defined as:\looseness=-1
    \begin{align}
        \pprefixwords(\wordsset) \defeq \sum_{\words \in \wordsset} p(\words)
    \end{align}
\end{defin}

Now, let $\circ$ denote concatenation (between either strings or sets of strings).
For instance, we can write $\words \circ \lexicon^* = \{\words \circ \words' \mathop{\mid} \words' \mathop{\in} \lexicon^*\}$ to represent the set of all strings with prefix $\words$.
We can compute our desired conditional distribution as the quotient of two evaluations of $\pprefixwords$:
\begin{align}\label{eq:quotient}
    p(\word \mid \words_{<t}) 
    &= \frac{\pprefixwords(\words_{< t} \circ \word \circ \lexicon^*)}
    {\pprefixwords(\words_{<t} \circ \lexicon^*)}
\end{align}
Note that this is a trivial invocation of the joint rule of probability: the conditional $p(\word \mid \words_{<t})$ is equal to the probability of observing prefix $\words_{<t} \circ \word$---represented by $\pprefixwords(\words_{< t} \circ \word \circ \lexicon^*)$---divided by the probability of observing prefix $\words_{<t}$---represented by $\pprefixwords(\words_{<t} \circ \lexicon^*)$.
We call probabilities of the form $\pprefixwords(\words \circ \lexicon^*)$ the \defn{prefix probability} of $\words$.
As we will show, careful consideration of these prefix probabilities is critical for converting between our desired distributions (over words) and the ones provided by language models (over subwords).

\paragraph{Orthography.}
We assume here this language can be written, and that it has a standardised orthographic convention.
Formally, given a language's alphabet $\alphabet$, each string $\words$ can be mapped to a sequence of characters $\characters \in \alphabet^*$ via function $\!\!\wordcharsmap\!\!: \lexicon^*\!\!\to\alphabet^*$.
Further, we assume this language allows for straightforward segmentation from orthography. 
Given a sequence of characters $\characters$, we can thus extract a sequence of words as $\charwordsmap(\characters) = \words$.

\vspace{-3pt}
\section{Subwords and Language Models}
\vspace{-2pt}

\mysubword{\unmmapedsubwords}{\vocab_{\texttt{x}}}

Most modern language models are not defined directly as distributions over words $\words$, but rather as distributions over \emph{sub}words. 
These subwords are themselves defined by a choice of \defn{tokeniser}.\footnote{We are not concerned with most aspects of individual tokenisers, and will focus on general considerations here. See \citet{Mielke2021BetweenWA} for a more comprehensive discussion.}
In this section, we first introduce tokenisers, and how they map words to subwords (and back). 
We then use these building blocks to show how we can compute word probabilities from subword probabilities.

\vspace{-2pt}
\subsection{From Words to Subwords and Back}\label{sec:subword_to_word}

We define a tokeniser here as a tuple $\langle\vocab, \subwordcharsmap, \charsubwordsmap\rangle$. 
This tuple consists of: (i) a \defn{vocabulary} $\vocab$, 
whose elements are subwords $\subword \in \vocab$, each of which represents a sequence of characters $\characters \in \alphabet^*$;%
\footnote{While subwords can be mapped back to a set of characters, they need not consist of only characters from the alphabet $\alphabet$. Additional markers---such as \bow---can be used.}
(ii) a \defn{detokenisation function} $\subwordcharsmap: \vocab^* \to \alphabet^*$, which is simply a function that maps a sequence of subwords to the characters they represent and concatenates them together; (iii) a \textbf{tokenisation function} $\charsubwordsmap: \alphabet^* \to \vocab^*$, which takes as input a character sequence and maps it to a subword sequence.
Notably, multiple subword sequences may map to the same character sequence. However, most tokenisers specify one of these subword sequences as the canonical mapping and employ a deterministic tokenisation function.

Collectively, the mapping functions we have defined give us the ability to convert between words and subwords, which will be necessary when using subword distributions to compute word probabilities.
We write word-to-subword mappings as:\looseness=-1%
\begin{align}\label{eq:mapping}
    \wordsubwordsmapmath \defeq \wordcharsmapmath \bullet \charsubwordsmapmath,\quad 
    \subwordwordsmapmath \defeq \subwordcharsmapmath\bullet\charwordsmapmath
\end{align}
where $\bullet$ represents function composition. 
Importantly, these functions reverse
each other when applied as $\subwordwordsmap(\wordsubwordsmap\left(\words\right)) = \words$, but not necessarily when applied in the opposite order.
The implication of this is that each $\words$ maps to a unique $\subwords$, and every $\words$ can be represented by some $\subwords$; but there are subword sequences that will \emph{not} be mapped to by our tokenisation function. 
For example, if a tokeniser maps word \wordexample{probability} to subwords $[\wordexample{\_prob}, \wordexample{ability}]$, then the subword sequence $[\wordexample{\_p}, \wordexample{r}, \wordexample{o}, \wordexample{b}, ...]$ will never be mapped to.
We denote \defn{unmapped subword sequences} as:
\begin{align}
    \unmmapedsubwords \defeq \vocab^*\setminus\left\{\wordsubwordsmapmath\!\!\left(\words\right) \mid \words \in \lexicon^*\right\}
\end{align}

\subsection{From Word to Subword Probabilities}

Now let $\ptheta$ be a language model with parameters $\vtheta$ and a vocabulary $\vocab$.
This model defines a probability distribution over the set of all finite subword sequences $\subwords \in \vocab^*$ and its parameters are optimized to provide good estimates of the true distribution over subwords, given by:
\begin{align}\label{eq:prob_subwords_from_words}
    p(\subwords) &= \sum_{\words \in \lexicon^*} p(\words)\, \one\left\{\subwords = \wordsubwordsmapmath(\words)\right\}
\end{align}
As not all subword sequences are mapped to, and because each mapping in $\wordsubwordsmap$ is unique, we can re-write this distribution as:
\begin{align} \label{eq:prob_subwords_from_words_with_zero}
    p(\subwords) = \left\{ 
    \begin{array}{cr}
         p(\words) &\quad  \texttt{if } \subwords = \wordsubwordsmapmath(\words) \\
         0 & \texttt{if } \subwords \in \unmmapedsubwords
    \end{array}
    \right.
\end{align}

\subsection{From Subword to Word Probabilities}

\Cref{eq:prob_subwords_from_words_with_zero} suggests a way to extract probabilities over words from a language model; we can simply use the equivalence:\footnote{
Notably, to apply this equivalence in practice, one needs an \defn{exact language model}, which we define as a model $\ptheta$ with the same support as $p$, i.e., 
$\ptheta(\subwords) = 0$ when $p(\subwords) = 0$.
Note that most neural language models \emph{cannot} assign zero probability to any subword sequence due to their use of a softmax projection in the final step of computing probabilities; they will thus not be exact in this sense.
While we focus on exact language models in this paper, we note that extending our results to inexact ones simply requires marginalising out potential ambiguities, i.e., computing $p(\words)$ for a given word requires summing over the (finite) set of subword sequences which map to it \citep{cao-rimell-2021-evaluate}.}
\begin{align}
    p(\words) &= p(\subwords),\quad \texttt{for}\,\, \subwords = \wordsubwordsmapmath(\words)\label{eq:word_to_subword_prob}
\end{align}

The implication of \cref{eq:word_to_subword_prob} is that
if we can create a subword set $\subwordsset$ that is ``equivalent'' to a chosen word set $\wordsset$, we would be able to compute $\wordsset$'s probability by summing over the subwords in $\subwordsset$.
Formally, we define the set equivalence $\stringequiv$  
between two sets of sequences as:
\begin{align}\label{eq:set_equivalence_defn}
    \wordsset\stringequiv\subwordsset
    \!\implies\!
    \left(\word \!\in\! \wordsset \!\!\iff\!\!\!\!\!\wordsubwordsmapmath\!\!(\words) \!\in\! \subwordsset \right)
\end{align}
Now let $\pprefixsubwords$ be a probability function defined analogously to $\pprefixwords$ (in \cref{defn:prob_words}).
It then follows that:
\begin{align}\label{eq:seq_equivalences}
    &\pprefixwords(\wordsset)  = \pprefixsubwords(\subwordsset), 
    \quad \texttt{for}\,\,
    \wordsset \stringequiv \subwordsset
\end{align}
We are now in a position to define our quantity of interest  $p(\word \mid \words_{<t})$ in terms of subword probabilities: it is simply the quotient of $\pprefixsubwords(\cdot)$ for two different sets $\subwordsset$.

\begin{lemma} \label{lemma:probability_word_from_lm}
    The contextual probability of a word can be computed using probability distributions over subwords as:
    \begin{align}\label{eq:quotient_lm}
        p(\word \mid \words_{<t}) 
        &= \frac{\pprefixsubwords(\subwordsset')}
        {\pprefixsubwords(\subwordsset'')}
    \end{align}
    where $\subwordsset' \stringequiv \words_{< t} \circ \word \circ \lexicon^*$ and 
    $\subwordsset'' \stringequiv \words_{<t} \circ \lexicon^*$.
\end{lemma}
\begin{proof}
    This result follows from a simple application of the equivalence in \cref{eq:seq_equivalences} to the definition of $p(\word \mid \words_{<t})$ in \cref{eq:quotient}. 
\end{proof}
Luckily, it is straightforward to find the sets $\subwordsset'$ and $\subwordsset''$ required by \cref{lemma:probability_word_from_lm}.
This is because, for a given word set $\wordsset$, the subword set 
\begin{align}
    \subwordsset = \left\{\wordsubwordsmapmath(\words) \mid \words \in \wordsset\right\}
\end{align}
satisfies $\wordsset\!\stringequiv\!\subwordsset$: first, by construction, we have that $\words\mathop{\in}\wordsset\!\mathop{\implies}\!\wordsubwordsmap(\words)\mathop{\in}\subwordsset$; second, due to the injectivity of $\wordsubwordsmap$, it must be that $\wordsubwordsmap\!(\words) \mathop{\in} \subwordsset\!\!\mathop{\implies}\!\!\words \mathop{\in} \wordsset$. 
These sets thus meet the \emph{iff} criteria required by our definition in \cref{eq:set_equivalence_defn}.

Before making use of \cref{eq:quotient_lm} for computing contextual probabilities, however, there is still one hurdle to overcome: the two sets $\wordsset' = (\words_{< t} \circ \word \circ \lexicon^*)$ and $\wordsset'' = (\words_{< t} \circ \lexicon^*)$ are infinite.
We must thus find a more efficient strategy to compute these probabilities than summing over the (also infinite) sets $\subwordsset'$ and $\subwordsset''$.

\subsection{Leveraging LMs' Autoregressiveness}

We now discuss how we can leverage the fact that most LMs compute probabilities autoregressively to efficiently compute the probabilities 
in \Cref{lemma:probability_word_from_lm}.
In short, most LMs provide estimates of conditional probabilities: $p(\subword \mid \subwords_{<t})$.
Given \cref{eq:quotient} and the fact that $\pprefixsubwords(\vocab^*) = 1$,
we can use these conditionals to compute prefix probabilities efficiently.
\begin{lemma}\label{lemma:chain_rule}
    Prefix probabilities can be computed using conditional probabilities as:
    \begin{align}\label{eq:prefix_equals_cond}
        \pprefixsubwords(\subwords \mathop{\circ} \vocab^*) 
        \mathop{=} \prod_{t=1}^{|\subwords|}\! \frac{\pprefixsubwords(\subwords_{<t} \mathop{\circ} \subword_t \mathop{\circ} \vocab^*)}{\pprefixsubwords(\subwords_{<t} \mathop{\circ} \vocab^*)}
        \mathop{=} \prod_{t=1}^{|\subwords|} p(\subword_{t} \mathop{\mid} \subwords_{<t})
    \end{align}
\end{lemma}

It follows that if we can find a set of subword sequences $\subwordsset = \{\subwords^{(k)}\}_{k=1}^{K}$ for which we have the equivalence 
$\words \circ \lexicon^* \stringequiv \bigcup_{\subwords \in \subwordsset} \subwords \circ \vocab^*$,
then we can compute prefix probabilities as:\footnote{In practice, we also need these prefix sets to be disjoint: $(\subwords \circ \vocab^*) \cap (\subwords' \circ \vocab^*) = \emptyset$ for $\subwords, \subwords' \in \subwordsset$. This will be the case whenever no $\subwords \in \subwordsset$ is a prefix of another $\subwords' \in \subwordsset$ (i.e., $\subwords' \notin \subwords\circ\vocab^*$). If there is an $\subwords$ which is a prefix of $\subwords'$, however, we can easily find a new set which still satisfies the equivalence above by dropping $\subwords'$ from $\subwordsset$. %
}
\begin{align}\label{eq:prefix_probs}
    \pprefixsubwords\Bigg(\bigcup\limits_{\subwords \in \subwordsset} \subwords \circ \vocab^* \Bigg) 
    &= \sum_{\subwords \in \subwordsset} \pprefixsubwords(\subwords \circ \vocab^*)
\end{align}
In turn, these let us compute $p(\word \mid \words_{<t})$ 
efficiently through \cref{eq:quotient_lm}.
For most tokenisers, finding a set $\subwordsset$ for which the equivalence 
$\words \circ \lexicon^* \stringequiv \bigcup_{\subwords \in \subwordsset} \subwords \circ \vocab^*$ holds is not actually possible due to the existence of unmapped sequences in $\subwords \circ \vocab^*$;
unmapped sequences, however, have zero probability and including them in $\subwordsset'$ or $\subwordsset''$ does not affect the equality in \cref{eq:quotient_lm}. 
We thus ignore this issue in our exposition, while still considering it in our theorem proofs.
We now outline tokeniser-specific considerations which influence how to choose these sets.\looseness=-1

\begin{figure*}
    \resizebox{0.94\linewidth}{!}{
    \begin{minipage}{\linewidth}
    \begin{align}
        &\wordsubwordsmapmath(\textwords{How} \circ \textwords{do} \circ \textwords{you} \circ \textwords{compute} \circ \textwords{a} \circ \textwords{word} \circ \textwords{'s} \circ \textwords{probability} \circ \textwords{?})
        \!\!\!\!\!\!\!\!\!\!\!\!\!\!\!\!\!
        \!\!\!\!\!\!\!\!\!\!\!\!\!\!\!\!\!\!\!\!\!\!\!\!\!\!\!\!\!\!\!\!\!\!\!\!
        \!\!\!\!\!\!\!\!\!\!\!\!\!\!\!\!\!\!\!\!\!\!\!\!\!\!\!\!\!\!\!\!\!\!\!\!
        \!\!\!\!\!\!\!\!\!\!\!
        & \mathcomment{\eow-marked, split punctuation} \nonumber \\
        &\begin{array}{ccccccccccccccccccc}
             & = & \wordsubwordsmaplargeequation(\textwords{How}) & \circlargeequation &  \wordsubwordsmaplargeequation(\textwords{do}) &\circlargeequation &\wordsubwordsmaplargeequation(\textwords{you}) & \circlargeequation &  \wordsubwordsmaplargeequation(\textwords{compute}) & \circlargeequation &  \wordsubwordsmaplargeequation(\textwords{a}) & \circlargeequation &  \wordsubwordsmaplargeequation(\textwords{word})  &  \circlargeequation  & \wordsubwordsmaplargeequation(\textwords{'s})  &  \circlargeequation  & \wordsubwordsmaplargeequation(\textwords{probability})  &  \circlargeequation  & \wordsubwordsmaplargeequation(\textwords{?}) \\
             & = & \texttokenised{How\_} & \circlargeequation &  \texttokenised{do\_} & \circlargeequation &  \texttokenised{you\_}&\circlargeequation  & \texttokenised{comp} \circ \texttokenised{ute\_} & \circlargeequation  & \texttokenised{a\_} & \circlargeequation  & \texttokenised{word} & \circlargeequation & \texttokenised{'s\_} & \circlargeequation  & \texttokenised{prob} \circ \texttokenised{ability} & \circlargeequation & \texttokenised{?}
        \end{array} \nonumber 
        \!\!\!\!\!\!\!\!\!\!\!\!\!\!\!\!\!\!\!\!\!\!\!\!\!\!\!\!\!\!\!\!\!\!\!\!
        \!\!\!\!\!\!\!\!\!\!\!\!\!\!\!\!\!\!\!\!\!\!\!\!\!\!\!\!\!\!\!\!\!\!\!\!
        \!\!\!\!\!\!\!\!\!\!\!\!\!\!\!\!\!\!\!\!\!\!\!\!\!\!\!\!\!\!\!\!\!\!\!\!
        \\
        ~ \nonumber\\
        &\wordsubwordsmapmath(\textwords{How} \circ \textwords{do} \circ \textwords{you} \circ \textwords{compute} \circ \textwords{a} \circ \textwords{word} \circ \textwords{'s} \circ \textwords{probability} \circ \textwords{?})
        \!\!\!\!\!\!\!\!\!\!\!\!\!\!\!\!\!\!\!\!
        \!\!\!\!\!\!\!\!\!\!\!\!\!\!\!\!\!\!\!\!\!\!\!\!\!\!\!\!\!\!\!\!\!\!\!\!
        \!\!\!\!\!\!\!\!\!\!\!\!\!\!\!\!\!\!\!\!\!\!\!\!\!\!\!\!\!\!\!\!\!\!\!\!
        \!\!\!\!\!\!\!\!\!\!\!
        & \mathcomment{\bow-marked, split punctuation} \nonumber \\
        &\begin{array}{ccccccccccccccccccc}
             & = & \wordsubwordsmaplargeequation(\textwords{How}) & \circlargeequation &  \wordsubwordsmaplargeequation(\textwords{do}) &\circlargeequation &\wordsubwordsmaplargeequation(\textwords{you}) & \circlargeequation &  \wordsubwordsmaplargeequation(\textwords{compute}) & \circlargeequation &  \wordsubwordsmaplargeequation(\textwords{a}) & \circlargeequation &  \wordsubwordsmaplargeequation(\textwords{word})  &  \circlargeequation  & \wordsubwordsmaplargeequation(\textwords{'s})  &  \circlargeequation  & \wordsubwordsmaplargeequation(\textwords{probability})  &  \circlargeequation  & \wordsubwordsmaplargeequation(\textwords{?}) \\
             & = & \texttokenised{How} & \circlargeequation &  \texttokenised{\_do} & \circlargeequation &  \texttokenised{\_you}&\circlargeequation  & \texttokenised{\_comp} \circ \texttokenised{ute} & \circlargeequation  & \texttokenised{\_a} & \circlargeequation  & \texttokenised{\_word} & \circlargeequation & \texttokenised{'s} & \circlargeequation  & \texttokenised{\_prob} \circ \texttokenised{ability} & \circlargeequation & \texttokenised{?}
        \end{array} \nonumber
        \!\!\!\!\!\!\!\!\!\!\!\!\!\!\!\!\!\!\!\!\!\!\!\!\!\!\!\!\!\!\!\!\!\!\!\!
        \!\!\!\!\!\!\!\!\!\!\!\!\!\!\!\!\!\!\!\!\!\!\!\!\!\!\!\!\!\!\!\!\!\!\!\!
        \!\!\!\!\!\!\!\!\!\!\!\!\!\!\!\!\!\!\!\!\!\!\!\!
    \end{align}%
    \end{minipage}
    }
    \caption{The output of tokenisers with different methods of handling word delineations.}
    \label{fig:tokenisers_output}
\end{figure*}

\section{The Nuances of Mapping: Tokeniser-dependent Strategies}\label{sec:per_tokeniser_strategies}

We are left with the task of finding a set of subword prefixes which will allow us to compute the probabilities of $\subwordsset' \stringequiv \wordsset'$ and $\subwordsset'' \stringequiv \wordsset''$.
In this section, we discuss how our tokeniser---specifically whether it uses end- or beginning-of-word markings in its vocabulary---affects this task.

\subsection{Segmentation-compatible Tokenisers}

In the following sections, we consider $\wordsubwordsmap$ that operate independently over words in a sequence $\words$.
This is necessary for our methods below, and is a common practice in NLP (typically called pre-tokenisation) where a text is segmented according to some criterion (e.g., white space) before being converted into subwords by a tokeniser.
Here, we consider pre-tokenisation to be one of the steps implemented by $\wordsubwordsmap$. 
We formalise this in the following definition.

\begin{defin}\label{defn:pretokenised_tokeniser}
    We define a \defn{segmentation-compatible tokeniser} as one whose operations can be decomposed across words in a sequence, i.e.:
    \begin{align} \label{eq:decompose_tokenisation}
        \!\!\!\wordsubwordsmapmath\!\!(\words) \! 
        &= \!\!\wordsubwordsmapmath\!\!(\words_{<t}) \mathop{\circ} \!\!\wordsubwordmapmath\!\!(\word_{t}) \mathop{\circ} \!\!\wordsubwordsmapmath\!\!(\words_{> t}) \\
        &= \!\!\wordsubwordmapmath\!\!(\word_1) \mathop{\circ} \!\!\wordsubwordmapmath\!\!(\word_2) \mathop{\circ} \cdots \mathop{\circ} \!\!\wordsubwordmapmath\!\!(\word_{|\words|}) \nonumber
    \end{align}
\end{defin}

While it is possible to create tokenisers with vocabularies in which subwords can cross word boundaries, the majority of them meet \cref{defn:pretokenised_tokeniser}.\footnote{E.g., the \texttt{sentencepiece} library \citep{kudo-richardson-2018-sentencepiece} has an option which allows multi-word subwords to be added to a tokeniser's vocabulary; by default, though, this option is disabled and it does not consider tokens of this format.}

The decomposition in \cref{defn:pretokenised_tokeniser} has an important implication.
As discussed in \cref{sec:subword_to_word}, the (sequence-level) tokenisation function $\!\wordsubwordsmap\!$ must be injective, meaning that each word sequence must map to a unique subword sequence; 
this, in turn, implies that concatenating the outputs of $\!\wordsubwordmap\!$ should always result in unique subword sequences.
This property is known in the compression literature as unique decodability 
\citep[page 105]{cover1991elements}.
At an intuitive level, we can see why this is a desirable property of a tokenisation function: when working with NLP models, we want to be able to deterministically map a sequence of subwords to a sequence of words.
A relatively simple strategy to ensure unique decodability, which is used by the majority of tokenisers, is to mark either the ends or beginnings of words (\eow or \bow) using a subset of the subwords in $\vocab$.
We discuss these strategies next.\looseness=-1

\newcommand{\charmapfuncwithoverscript}[2]{\underset{\scaleto{#1}{4pt}}{\charmap^{#2}}}
\newcommand{\wordsubwordmapeow}{\charmapfuncwithoverscript{\lexicon\to\vocab^*}{\eow}}
\newcommand{\wordsubwordmapbow}{\charmapfuncwithoverscript{\lexicon\to\vocab^*}{\bow}}
\newcommand{\wordsubwordmapmid}{\charmapfuncwithoverscript{\lexicon\to\vocab^*}{\midstring}}

\subsection{End of Word Markings}

We now consider \eow-marking tokenisers.
These tokenisers use a subset of their vocabulary $\vocabeow \subseteq \vocab$ to indicate the end of words,\footnote{The case of $\vocabeow = \vocab$ or $\vocabbow = \vocab$ happens when $\vocab=\lexicon$; while possible in theory, it will not happen in practice since a language model cannot have an infinite vocabulary.\label{footnote}}
with the rest of the vocabulary $\vocabmid \defeq \vocab \setminus\vocabeow$ mapping back to the beginning or middle of words.

\begin{defin}\label{defn:eow_tokeniser}
    An \defn{\eow-marking tokeniser} is a segmentation-compatible tokeniser which marks ends of words.
    Its word-level tokenisation function can be written as $\wordsubwordmapeow: \lexicon \to \vocabmid^{*} \circ \vocabeow$.\footnote{Note that only subword sequences ending in $\subword \in \vocabeow$ or the empty sequence (i.e., $\emptystring$) are valid under this tokeniser.
    This is because: $
    \bigcup_{i=0}^{\infty} (\vocabmid^{*}\mathop{\circ}\vocabeow)^{i} \subseteq
    \{\emptystring\}\mathop{\cup}(\vocab^{*}\mathop{\circ}\vocabeow)$.
    }
\end{defin}

Importantly, given the definition above,
when a subword $\subword_{t} \in \vocabeow$ is observed, 
it means that the current subsequence $\subwords_{\scaleto{t':t}{5pt}}$ (where $t'\!\leq\!t$) 
can be mapped back to a word, and that a subsequence representing a new word will begin at $\subword_{\scaleto{t+1}{5pt}}$.
(The current subsequence $\subwords_{\scaleto{t':t}{5pt}}$ is thus determined by the smallest $t'$ for which $\subwords_{\scaleto{t':t-1}{5pt}}\!\in\!\vocabmid^*$; note that this means either $t'=1$ or $\subword_{\scaleto{t'-1}{5pt}} \in \vocabeow$.)  
This property implies that \eow-marking tokenisers provide \defn{instantaneous decodability}
\citep[page 106]{cover1991elements}:
prefix $\subwords_{\scaleto{\leq t}{5pt}}$ with $\subword_{t} \in \vocabeow$ is instantaneously decodable, as it always maps to the same words, regardless of its continuation $\subwords_{\scaleto{> t}{5pt}}$.
Instantaneous decodability allows us to compute the contextual probability of a word as follows.

\begin{restatable}{theorem}{probofeowtokeniser} \label{thm:probofeowtokeniser}
    Let $\wordsubwordsmap$ be a \eow-marking tokeniser.
    Further, let $\subwordschosen \defeq \wordsubwordsmap(\words)$.
    We can show the following equivalence:
    \begin{align}\label{eq:eow_equivalence}
        &\pprefixwords(\words_{<t} \circ \lexicon^*) = \pprefixsubwords(\subwordschosenprefix \circ \vocab^*) \\
        &\pprefixwords(\words_{<t} \circ \word \circ \lexicon^*) = \pprefixsubwords(\subwordschosenprefix \circ \subwordschosenword \circ \vocab^*) \nonumber
    \end{align}
    Further, we can compute a word's probability as:
    \begin{align} \label{eq:prob_eow_tokeniser}
        p(\word \mid \words_{<t}) 
        &= \underbrace{\prod\limits_{t'=1}^{\left|\subwordschosenword\right|} p\left(\subwordschosenword_{t'} \mid \subwordschosenprefix \circ \subwordschosenword_{<t'}\right)}_{p\left(\subwordschosenword \mid \subwordschosenprefix\right)}
    \end{align}
\end{restatable}
\begin{proof}
    See \cref{app:proof_theorem_probofeowtokeniser} for formal proof.
\end{proof}

\Cref{eq:eow_equivalence} follows from instantaneous decodability, as every sequence $\subwords \in \subwordschosen\circ\vocab^*$ 
maps back to $\words \circ \lexicon^*$.
\Cref{eq:prob_eow_tokeniser} then follows from a simple application of \cref{lemma:probability_word_from_lm,lemma:chain_rule}:
\begin{align}
    p\left(\subwordschosenword \mid \subwordschosenprefix\right) 
    = \frac{\prod_{t'=1}^{\left|\subwordschosenfull\right|} p\left(\subwordschosenfull_{t'} \mid \subwordschosenfull_{<t'}\right)}{\prod_{t'=1}^{\left|\subwordschosenprefix\right|} p\left(\subwordschosenprefix_{t'} \mid \subwordschosenprefix_{<t'}\right)}
\end{align}
Notably, \cref{eq:prob_eow_tokeniser} is fairly straightforward and is how most NLP practitioners would compute a word's probability.
In the next section, however, we see that it would not compute the correct probabilities if using \bow-marking tokenisers.

\begin{figure*}[t]
    \centering
\begin{tcolorbox}[colback=cyan!5!white,colframe=cyan!75!black,left=4pt,title=\!\!Worked Example: {\small Contextual probability computation for word  \textwords{mark} using a \bow-marking tokeniser}]
    \vspace{-5pt}
    {\small
    \textcolor{gray}{Let ``$\textwords{She saw the mark...} $'' be our context of interest; we thus have $\words = \langle  \textwords{She},  \textwords{saw},  \textwords{the},  \textwords{mark},...\rangle$. 
    Further, let $\ptheta$ be our language model with vocabulary:}
    \vspace{-5pt}
    \begin{align*}
         \vocab = \big\{\texttokenised{\_a}, \texttokenised{\_an}, \texttokenised{\_mark}, \texttokenised{\_saw}, \texttokenised{\_She},  \texttokenised{\_the},  \texttokenised{er}, \texttokenised{tion}, \texttokenised{ing}, \texttokenised{ed}\big\}
     \end{align*}
     \textcolor{gray}{Let's assume that we are interested in estimating  $p(\textwords{mark} \mid  \langle  \textwords{She},  \textwords{saw},  \textwords{the}\rangle)$ using $\ptheta$. To employ \cref{eq:quotient_lm}, we must compute $\pprefixsubwords(\subwordsset)$ for} 
     $\subwordsset' \stringequiv \langle  \textwords{She},  \textwords{saw},  \textwords{the}\rangle \circ \textwords{mark} \circ \lexicon^*$ \textcolor{gray}{and} 
     $\subwordsset'' \stringequiv \langle  \textwords{She},  \textwords{saw},  \textwords{the}\rangle \circ \lexicon^*$.
     \textcolor{gray}{For vocabularies derived using \bow-marking tokenisers, \cref{thm:probofbowtokeniser} states that we should use:} 
     \vspace{-5pt}
     \begin{align*}
        &\subwordsset' = \texttokenised{\_She}\circ \texttokenised{\_saw}\circ \texttokenised{\_the}\circ \texttokenised{\_mark}
        \circ\, \vocabboweospluscont \\
        &\subwordsset'' = \texttokenised{\_She}\circ \texttokenised{\_saw}\circ \texttokenised{\_the}
        \circ\, \vocabboweospluscont
    \end{align*}
    \vspace{-5pt}
    \textcolor{gray}{where $\vocab_{\texttt{bow}} = \{\texttokenised{\_a}, \texttokenised{\_an}, \texttokenised{\_mark}, \texttokenised{\_saw}, \texttokenised{\_She},\texttokenised{\_the}\}$.
    Using this theorem's \cref{eq:prob_bow_tokeniser} we arrive at:}
    \vspace{-5pt}
    \begin{align*}
        p(\textwords{mark} \mid \! \langle  \textwords{She},  \textwords{saw},  \textwords{the}\rangle) = 
        \ptheta(\texttokenised{\_mark} &\mid\!  \langle  \texttokenised{\_She},  \texttokenised{\_saw},  \texttokenised{\_the}\rangle) \cdot\frac{\!\sum_{\{\subword\in\vocabboweos\}}\! \ptheta\left(\subword \mathop{\mid}  \langle  \texttokenised{\_She},  \texttokenised{\_saw},  \texttokenised{\_the}\rangle \circ\texttokenised{\_mark}\right)}
        {\sum_{\{\subword\in\vocabboweos\}}\! \ptheta\left(\subword \mathop{\mid}  \langle  \texttokenised{\_She},  \texttokenised{\_saw},  \texttokenised{\_the}\rangle\right)}
    \end{align*}
    \textcolor{gray}{Note that this computation specifically discounts the probabilities $p(\textwords{marker} \!\mid\!  \langle  \textwords{She},  \textwords{saw},  \textwords{the}\rangle)$, $p(\textwords{marktion} \!\mid\!  \langle  \textwords{She},  \textwords{saw},  \textwords{the}\rangle)$, $p(\textwords{markerer} \!\mid \! \langle  \textwords{She},  \textwords{saw},  \textwords{the}\rangle)$, etc., which otherwise would have incorrectly counted towards our estimate of $p(\textwords{mark} \mid \! \langle  \textwords{She},  \textwords{saw},  \textwords{the}\rangle)$.}
    }
\end{tcolorbox}
    \vspace{-10pt}
    \caption{
    Example for computing a word's probability using a LM over subwords defined by a \bow-marked tokeniser.
    }\label{fig:example}
    \vspace{-10pt}
\end{figure*}

\subsection{Beginning of Word Markings}\label{sec:bow}

We now consider \bow-marking tokenisers.
Analogously to the \eow case, a subset of a \bow-marking tokeniser's vocabulary $\vocabbow\subseteq \vocab$  is used exclusively to indicate word beginnings.
The rest of the vocabulary $\vocabmid \defeq \vocab \setminus\vocabbow$ then represents either the middle or end of words.
We provide a formal definition of this tokeniser below.\looseness=-1

\begin{defin}\label{defn:bow_tokeniser}
    A \defn{\bow-marking tokeniser} is a segmentation-compatible tokeniser which marks beginnings of words.
    Its word-level tokenisation function is written as 
    $\wordsubwordmapbow: \lexicon \mathop{\to} \vocabbow \mathop{\circ} \vocabmid^{*}$.\footnote{Similarly to with \eow, not all subword sequences are valid under \bow tokenisers, only sequences in $\{\emptystring\} \mathop{\cup} (\vocabbow \circ \vocab^{*})$.}
\end{defin}

Given the definition above, when a subword $\subword_{t} \in \vocabbow$ is observed, it thus means that a \emph{previous} subsequence $\subwords_{\scaleto{t':t-1}{5pt}}$
can be mapped back to a word, and that a subsequence representing a new word begins at $\subword_{\scaleto{t}{5pt}}$.
(The previous subsequence $\subwords_{\scaleto{t':t-1}{5pt}}$ is determined by $\subword_{\scaleto{t'}{5pt}} \in \vocabbow$ and $\subwords_{\scaleto{t'+1:t-1}{5pt}} \in \vocabmid^*$.)
Such tokenisers are thus \emph{not} instantaneously decodable. 
They only provide what we term \defn{near-instantaneous decodability}:
a prefix $\subwords_{\scaleto{\leq t}{5pt}}$ 
does \emph{not} always map to the same words, as its mapping depends on whether the following subword $\subword_{\scaleto{t+1}{5pt}}$ is in $\vocabbow\cup\{\eos\}$.\footnote{
Here, we define the concatenation of any sequence with \eos to be itself, e.g., $\subwords\circ\eos = \subwords$.
} 
Computing probabilities with near-instantaneous codes thus requires discounting the probability of continuations $\subword_{\scaleto{t+1}{5pt}} \notin \vocabbow\cup\{\eos\}$; we label this discount factor as \newbugfixcounter{bugfix:bow}.

\begin{restatable}{theorem}{probofbowtokeniser} \label{thm:probofbowtokeniser}
    Let $\wordsubwordsmap$ be a \bow-marking tokeniser.
    Further, let 
    $\overline{\,\cdot\,}$ represent the union of a set with $\eos$, e.g.,
    $\vocabboweos = \vocabbow \cup \{\eos\}$.
    We can show the following equivalence:
    \begin{align}\label{eq:bow_prob_equivalence}
        &\pprefixwords(\words_{<t} \circ \lexicon^*) = \pprefixsubwords(\subwordschosenprefix \circ \vocabboweospluscont) \\
        &\pprefixwords(\words_{<t} \circ \word \circ \lexicon^*) = \pprefixsubwords(\subwordschosenprefix \circ \subwordschosenword \circ \vocabboweospluscont) \nonumber
    \end{align}
    Further, we can compute a word's probability as:
    \begin{align} \label{eq:prob_bow_tokeniser}
        &p(\word \mid \words_{<t}) = \\
        &\,\underbrace{
        \prod\limits_{t'=1}^{\left|\subwordschosenword\right|}\! p\left(\subwordschosenword_{t'} \mathop{\mid} \subwordschosenprefix \mathop{\circ} \subwordschosenword_{<t'}\right)}_{p\left(\subwordschosenword \mathop{\mid} \subwordschosenprefix\right)}
        \underbrace{
        \frac{\!\sum_{\{\subword\in\vocabboweos\}}\! p\left(\subword \mathop{\mid} \subwordschosenprefix\!\circ\!\subwordschosenword\right)}
        {\sum_{\{\subword\in\vocabboweos\}}\! p\left(\subword \mathop{\mid} \subwordschosenprefix\right)}}_{\text{\usebugfixcounter{bugfix:bow}}} \nonumber
    \end{align}
\end{restatable}
\begin{proof}
    See \cref{app:proof_theorem_probofbowtokeniser} for formal proof.
\end{proof}

\Cref{eq:bow_prob_equivalence} follows from near-instantaneous decodability, as
every sequence $\subwordschosen\circ\vocabboweospluscont$ maps back to $\words \circ \lexicon^*$,
but sequences in $\subwordschosen\circ\vocabmid\circ\vocab^*$ do not.
\Cref{fig:tokenisers_output} contains an example of a sequence tokenised using either \eow- or \bow-marking tokenisers; \Cref{fig:example} contains an example motivating \usebugfixcounter{bugfix:bow}.

\subsection{Practical Concerns and Corner Cases}\label{sec:corners}

In this section, we discuss corner cases that deserve special consideration.
Many of these cases arise because of practical demands, e.g., ensuring the presence or absence of white space where appropriate. %
Notably, the need for these corner cases is often language-dependent, as they arise due to orthographic conventions. 
We discuss the implications of two tokeniser conventions that handle special cases: the treatment of the beginnings and ends of sequences. 

\newcommand{\charpunct}{\characterspace_{\texttt{!?}}}
\newcommand{\vocabpunct}{\vocab_{\texttt{!?}}}
\newcommand{\vocabpuncteos}{\vocabeos_{\texttt{!?}}}
\paragraph{Non-\eow-marked Final Words.}
Several \eow-marking tokenisers do not decompose exactly as in \cref{eq:decompose_tokenisation}, but treat the final word in a sequence differently.
Specifically, they override the behaviour of $\wordsubwordmap$ on these words and do \emph{not} use subwords from $\vocabeow$ to mark its ends.
This is also often the treatment applied to words followed immediately by punctuation.
This mechanism allows tokenisers to avoid implying the existence of a white space that does not exist, e.g., at the end of a string.
Notably, this breaks instantaneous decodability, making this code only near-instantaneous.
A simple example demonstrates this fact: let $\subwordschosenword_{\midstring} \defeq \wordsubwordmapmid(\word)$, where $\wordsubwordmapmid: \lexicon \to \vocabmid^*$. 
Upon observing subsequence $\subwordschosenword_{\midstring}$, we cannot instantaneously map it back to $\word$, and must wait for the next symbol:
if $\subwordschosenword_{\midstring}$ is followed by either \eos or punctuation, then it is mapped back to $\word$; 
if not, it is mapped to another word.
Handling this thus requires the following fix (termed \newbugfixcounter{bugfix:eowateos} here):
\resizebox{\columnwidth}{!}{
    \begin{minipage}{\linewidth}
\vspace{-10pt}
\begin{align} \label{eq:prob_eoweos_tokeniser}
    &p(\word \mid \words_{<t}) = \\
    &
    \bigg(\!
    p\left(\subwordschosenword_{\midstring} \!\mathop{\mid} \subwordschosenprefix\right) \!\! \underbrace{\sum_{\subword \in \vocabpuncteos} \!\! p(\subword\!\mathop{\mid} \subwordschosenprefix\!\mathop{\circ} \subwordschosenword_{\midstring})
    \!\bigg) \!+\! 
    p(\subwordschosenword\!\mathop{\mid} \subwordschosenprefix)
    }_{\text{\usebugfixcounter{bugfix:eowateos}}} \nonumber
\end{align}
\end{minipage}}

\begin{figure*}[t]
    \centering
    \includegraphics[width=\textwidth]{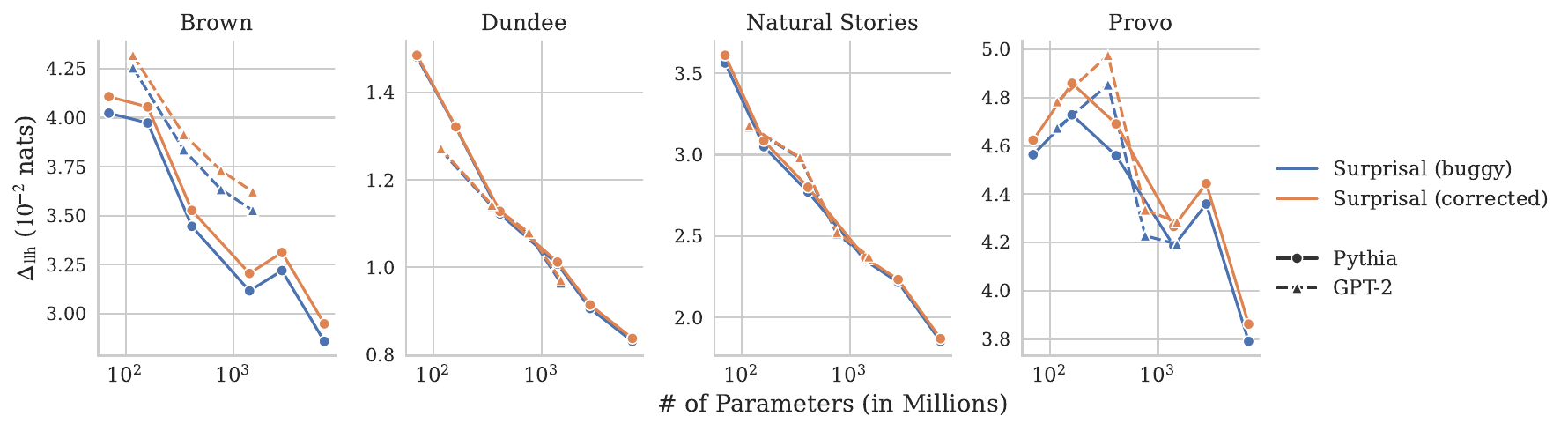}
    \vspace{-18pt}
    \caption{$\deltallh$ between regressors with and without surprisal as a predictor. We include $\deltallh$ when using surprisal estimates computed from language models across several sizes and families.  Results are presented both when using the buggy and correct methods for surprisal estimation.}\label{fig:llh}
    \vspace{-5pt}
\end{figure*}

\vspace{-3pt}
\paragraph{Non-\bow-marked First Words.}
Just as \eow-marking tokenisers often treat final words differently, \bow-marking tokenisers treat the first word in a sequence differently to handle white space appropriately.
These tokenisers typically do \emph{not} mark first words with \bow, and instead apply  $\wordsubwordmapmid$ to  $\word_1$. 
This affects the probability computation of the first word in a sequence. 
In such cases, the prefix $\words_{<t}$ of the first word is empty (denoted here as $\emptystring$). 
While computing a word's contextual probability according to \cref{eq:bow_prob_equivalence} requires computing $\pprefixsubwords(\vocabboweospluscont)$, the first subword in a sequence will not be in $\vocabbow$, but in $\vocabmid$ instead.
The probability computation of such words thus requires the following correction (\newbugfixcounter{bugfix:bowatbos}):
\vspace{-12pt}
\begin{align} \label{eq:prob_bowatbos_tokeniser}
    &p(\word \mid \emptystring) = 
    p\left(\subwordschosenword_{\midstring} \mid \emptystring \right)\,
    \underbrace{\frac{\sum_{\{\subword \in \vocabboweos \}} p\left(\subword \mid \subwordschosenword\right)}
    {\sum_{\{\subword \in \vocabmideos\}} p\left(\subword \mid \emptystring \right)}}_{\text{\usebugfixcounter{bugfix:bowatbos}}}
\end{align}

\subsection{Defining \texorpdfstring{$\vocabbow$ and $\vocabeow$}{\bow and \eow Vocabularies}}
\label{sec:vocab_definitions}

Defining sets $\vocabbow$ or $\vocabeow$ for a given tokeniser is not necessarily straightforward,
as tokenisers do not explicitly mark $\bow$ or $\eow$ in their vocabularies.%
\footnote{They often mark white spaces instead (denoted here as $\textchar{\_}$), but white space need not be the only word-boundary marker.}
Further, these sets' definitions will depend on what a researcher considers a word to be.
As an example, we use the sentence in \cref{fig:tokenisers_output}: \textchar{How\_do\_you\_compute\_a\_word's\_probability?}.
One could define words to be the set of whitespace-separated character sequences: 
$\textwords{How} \circ \textwords{do} \circ \textwords{you} \circ \textwords{compute} \circ \textwords{a} \circ \textwords{word's} \circ \textwords{probability?}$.
However, one may also consider punctuation and clitics to impose word boundaries, meaning words would instead be delineated as:
$\textwords{How} \circ \textwords{do} \circ \textwords{you} \circ \textwords{compute} \circ \textwords{a} \circ \textwords{word} \circ \textwords{'s} \circ \textwords{probability} \circ \textwords{?}$. 
In the former case, we would define $\vocabbow$ and $\vocabeow$ simply as the set of subwords with a leading or trailing white space (e.g., $\texttokenised{\_word} \in \vocabbow$ or $\texttokenised{'s\_} \in \vocabeow$). 
In the latter case though, subwords starting with punctuation or clitics should also be included in $\vocabbow$ (e.g., we require $\texttokenised{'s} \in \vocabbow$). 
This choice further impacts probability computations:
computing $\eow$-marking probabilities in the former case simply requires \cref{eq:prob_eow_tokeniser}, while in the latter case it requires \usebugfixcounter{bugfix:eowateos}.\footnote{
A recent work \citep{giulianelli-etal-2024-proper} proposes a method allowing the computation of the probability of any character span within a sequence. 
E.g., one can compute the probability of \textchar{word} or \textchar{ord's\_prob} in the example above.
While computing the probability of arbitrary character spans can be valuable, we note that there is no single sequence of characters that is equivalent to a word.
For example, $p(\textchar{compute\_} \!\mid\! \characters_{<t})$  is the probability of $\textchar{compute}$ followed by $\textchar{\_}$; the methods here, however, compute the probability of a word, $p(\textwords{compute} \!\mid\! \words_{<t})$, regardless of what follows it.
We can combine our considerations and their method to recover the probability of a word by first defining a set of word-ending characters $\characterspacebow$, and then using it to marginalise a word's probability over possible word-ending continuations:
$p(\textwords{compute} \!\mid\! \words_{<t}) \! = \!\sum_{\character \in \characterspaceboweos}\! p(\textchar{compute} \circ \character \!\mid\! \characters_{<t})$.
We thus see our works as complementary.\looseness=-1
}\looseness=-1
\section{Experiments}

We now investigate how correcting the computation of word probability estimates affects the results of prior studies. 
These works incorrectly computed probabilities as $p\left(\subwordschosenword \mathop{\mid} \subwordschosenprefix\right)$ (i.e., using \cref{eq:bow_prob_equivalence} without \usebugfixcounter{bugfix:bow}), which we term as \defn{buggy} estimates here. 
We explore two settings: psycholinguistics experiments surrounding sentence comprehension \citep{hale2001probabilistic,levy2008expectation} and computational linguistics experiments assessing the lexicon's communicative efficiency \citep{piantadosi2011word,gibson2019efficiency}.
We follow these works' experimental methodologies, observing how the use of corrected surprisal estimates impacts the conclusions that were originally drawn. 

\paragraph{Models.}
In our first experiment, we estimate contextual probabilities using GPT-2 \cite{gpt2} and Pythia \cite{biderman-etal-2023-pythia}; in the second, we focus only on Pythia.
Both these suites contain language models of various sizes.
We use these models' open-source versions from the \texttt{transformers} library \citep{wolf-etal-2020-transformers}. 
GPT-2 and Pythia use \bow-marking tokenisers, meaning we employ the methods discussed in \cref{sec:bow} to compute words' contextual probabilities.

\subsection{Sentence Comprehension} \label{sec:surprisal_theory}

\newcommand{\surp}{h}

Surprisal theory \citep{hale2001probabilistic,levy2008expectation} hypothesises that readers keep a belief distribution over meanings while reading; after observing each word in a sentence, they must thus update this distribution.
Under some assumptions about how these belief updates are performed, surprisal theory then predicts that their cost is related to a word's \defn{surprisal}, defined as the negative log-probability:
\begin{align}\label{eq:surprisal}
    \surp(\word_t) \defeq - \log p(\word_t \mid \words_{<t})
\end{align}
Surprisal theory is widely accepted as a model of comprehension effort, with numerous works empirically supporting it \citep[\textit{inter alia}]{smith2008optimal,smith2013-log-reading-time,goodkind-bicknell-2018-predictive,shain-2019-large,wilcox2020predictive,wilcox-etal-2023-language,oh2022comparison,shain2024large}.
Notably, the true contextual probabilities $p(\word_t \mid \words_{<t})$ required to compute surprisal are unknown, and must be approximated.
All of the works above use language models to do so, with the most recent using LMs which operate on top of subwords produced by $\bow$-marking tokenisers \citep[e.g.,][]{oh2023why,oh-schuler-2023-transformer,shain2024large,pimentel2023effect}.
Notably, these works compute surprisal estimates using the aforementioned buggy estimates of $p(\word_t \mid \words_{<t})$.
In this section, we reproduce some of these prior works' results, observing how this correction affects results.

\paragraph{Setup Summary.}
We run our analyses on 4 reading times datasets---Brown, Dundee, Natural Stories, and Provo. 
Further, following prior work \citep{wilcox-etal-2023-language,oh2023why}, we evaluate surprisal's predictive power over reading times by measuring the change in data log-likelihood $\deltallh$ when using linear regressors with and without surprisal as a predictor.
More details about our experimental setup are in \cref{app:setup_sentence_comprehension}.

\paragraph{Results.}
\cref{fig:llh} shows the change in data log-likelihood under regressors with and without surprisal as a predictor; values are detailed in \cref{tab:surp_theory_full_table} (in the appendix). 
We first note that the predictive power of surprisal decreases as language model size increases, as observed in prior work \cite{oh2023why,shain2024large}.
Here however, we are more interested in the effect of our corrections on these results---labelled as buggy vs. corrected surprisal.
Interestingly, we observe only small changes in predictive power due to our correction; individually, these changes are only significant for a few models (see \cref{tab:surp_theory_full_table} for detailed results).
However, when analysed in aggregate for all models, we see this positive improvement is consistent and significant in the four dataset ($\alpha<0.01$ in our permutation tests).
We also confirm the same patterns in seven other languages in \cref{app:multilingual_surprisal}.

\subsection{Communicative Efficiency}

Languages' lexicons have been studied for decades in an effort to gain better insights about the forces that shape natural languages \citep{zipf1935psychobiology,howes1968zipf,bentz2016zipf,levshina2022frequency}.
One characteristic of particular interest has been word lengths and how a tendency for communicative efficiency has influenced them.
There are several hypotheses about the exact way in which this tendency takes effect.
\citet{zipf1935psychobiology} argues that speakers have a tendency towards minimising utterance lengths,
and therefore that word lengths should correlate with frequencies.
\citet{piantadosi2011word} argues that speakers maximise information transfer,
and thus word lengths should correlate with a word's expected surprisal instead:
\begin{align}\label{eq:cch_piantadosi}
    \expectsurp \defeq \expect_{\words_{< t}}\left[- \log p(\word_t \mid \words_{<t}) \mid \word_t \right]
\end{align}
We follow \citet{pimentel-etal-2023-revisiting} in calling this the \defn{channel capacity hypothesis} (CCH). 
Finally, \citet{pimentel-etal-2023-revisiting} point out an issue with \citeauthor{piantadosi2011word}'s solution, and argue that to maximise information transfer, lengths should correlate with the following value instead:\footnote{See their paper for a derivation for this fix.}
\begin{align}\label{eq:cch_pimentel}
    &\expectsurpsquared \!\defeq\! 
    \frac{\expect_{\words_{< t}}\!\left[(- \log p(\word_t \!\mid \words_{<t}))^2 \!\mid \word_t \right]}
    {\expect_{\words_{< t}}\!\left[- \log p(\word_t \!\mid \words_{<t}) \!\mid \word_t \right]}
\end{align}

\paragraph{Setup Summary.}
We run our analysis using a subset of the English portion of Wiki-40B \citep{Guo_Dai_Vrandečić_Al-Rfou_2020}.
We compare the three values above (unigram frequency, and \cref{eq:cch_piantadosi,eq:cch_pimentel}); evaluating them based on their correlation with words' lengths.
Two of these values depend on a word's contextual probability, and we thus also compare their fixed vs.\ buggy versions.

\paragraph{Results.}
The results in \cref{fig:word_lengths} confirm the findings of \citet{pimentel-etal-2023-revisiting}:
once larger (and better) language models are used to estimate words' surprisals, the metrics under the CCH hypothesis (both \citeauthor{piantadosi2011word}'s and \citeauthor{pimentel-etal-2023-revisiting}'s versions) become weaker predictors of word lengths.
Interestingly, correcting the computation of surprisals also leads to a drop in the correlations between CCH predictors and word lengths.
Improving CCH's predictors thus consistently hurts its predictive power over word lengths---either when using better models, \citeauthor{pimentel-etal-2023-revisiting}'s fix to CCH's optimal solution, or our fix to probability computations.
We conclude, as \citeauthor{pimentel-etal-2023-revisiting}, that word lengths are best predicted by Zipf's hypothesis.

\begin{figure}
    \centering
    \includegraphics[width=.95\linewidth]{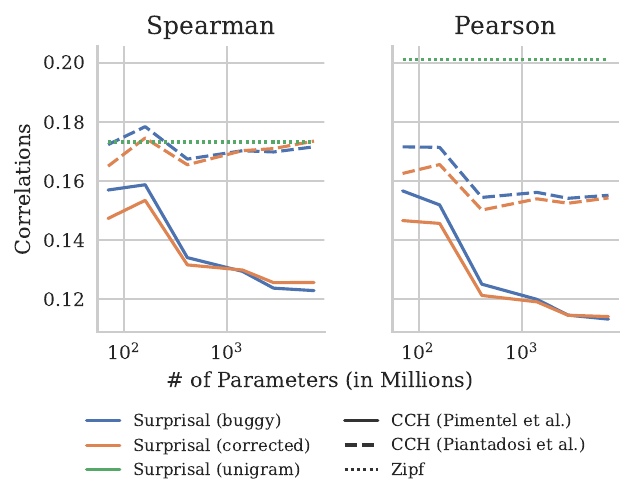}
    \vspace{-10pt}
    \caption{Correlation between English word lengths and the values predicted by either a Zipfian notion of optimality, or the channel capacity hypothesis.
    CCH (\citeauthor{pimentel-etal-2023-revisiting}) and CCH (\citeauthor{piantadosi2011word}) refer to \cref{eq:cch_piantadosi,eq:cch_pimentel}.}\label{fig:word_lengths}
    \vspace{-12pt}
\end{figure}

\section{Conclusion}
This work expounds on the intricacies of accurately computing contextual word probabilities using language models. 
We focus on the challenges posed by the use of subword vocabularies.
We show that subword vocabularies defined using beginning-of-word (\bow) tokenisers---common in many modern LMs---introduce complexities that are often overlooked. 
We point out that this has led to potential inaccuracies in computing probability estimates of various prior empirical analyses.
Our methodological corrections lead to significant differences in results, although the overarching conclusions of the previous studies that we explore remain the same. 
This finding underscores the importance of precise computational methods in linguistic research.
Future work should ensure these corrections are adopted to enhance the reliability of their analyses.

\section*{Limitations}
The authors see limitations in both the theoretical and empirical aspects of this work. 
Perhaps the main theoretical limitation is the lack of consideration of all potential corner cases which tokenisers might implement (similar to, e.g., those discussed in \cref{sec:corners}).
The use of white space differs from language to language, and many corner cases of tokeniser behaviour are designed specifically to handle this.
There are likely other fixes to probability computations that would need to be derived to handle paradigms not discussed in \cref{sec:corners}.
In Spanish, for instance, words following  ``¿'' are usually not \bow-marked, and might thus require the use of an approach similar to \usebugfixcounter{bugfix:bowatbos}.
Our theoretical results are also limited to autoregressive models. 
While the majority of today's language models meet this criterion, it is feasible that future language models would be designed differently and consequently, our methods would no longer be necessarily applicable. 
On the empirical side, a large limitation of our work is the exploration of the impact of our methods in only two studies. 
Additional studies are thus needed to understand the full extent to which our corrections impact empirical results in other areas of computational linguistics (and of NLP, more broadly).

\section*{Acknowledgments}

We thank Ethan Wilcox for many discussions about this paper, and for helping to draft parts of it.
We also thank Sotiris Anagnostidis and Pietro Lesci for feedback on earlier versions of this manuscript, and Yahya Emara and Mario Giulianelli for feedback on the final version. 
\bibliography{custom}

\onecolumn
\appendix

\section{Experimental Setup}

\subsection{Sentence Comprehension} \label{app:setup_sentence_comprehension}

\paragraph{Data.}
We use four well-established reading time datasets, in which 
participants were given text passages to read and their reading time was recorded.
For two of these datasets---Natural Stories \citep{futrell-etal-2018-natural} and Brown \citep{smith2013-log-reading-time}---measurements were collected using the self-paced paradigm \citep{Just1982ParadigmsAP}.
For the other two datasets---Provo \citep{provo} and Dundee \citep{dundee}---eye-tracking movements were recorded.
Each of these datasets provides the reading time each participant spent on a word.
Following the works whose experiments we aim to replicate, we aggregate reading times per word (i.e., across participants).
We thus analyse the average reading time participants spent on a word.

\paragraph{Evaluation.}
Studies of sentence comprehension are often concerned with a variable's \defn{predictive power}: its ability to predict sentence comprehension data.
Formally, let $\dataset = \{\vx_n, y_n\}_{n=1}^N$ be a reading times dataset, 
where
$y_n \in \R_{+}$ represents the average time participants spent reading a word $\word_n$, and
$\vx_n \in \R^{d}$ is a vector containing a number of measurements taken on that word.
Among these quantities is a word's length (in characters) and unigram frequency.
Further, let $\regressor$ be a regressor that takes $\vx_n$ as input and predicts $y_n$.
We use $\psi$ to denote this regressor's parameters.
A variable's predictive power is then the change in $\dataset$'s log-likelihood (denoted as $\deltallh$) under two regressors: one where $\vx$ includes this variable ($\regressorone$), and one where it does not ($\regressortwo$):
\begin{align}
    \deltallh \defeq \mathrm{llh}(\regressorone, \dataset) - \mathrm{llh}(\regressortwo, \dataset)
\end{align}
Here, we use this equation to measure surprisal's predictive power.
Further, we estimate this change in data log-likelihood (denoted as $\deltallh$) using 10-fold cross-validation, and we leverage these results to run paired permutation tests.
Finally, we account for spillover effects by including features of word $\word_n$ as well as its three preceding words in $\vx$.

\subsection{Communicative Efficiency} \label{app:setup_comm_efficiency}

We largely follow the setup of \citet{pimentel-etal-2023-revisiting}.
We highlight the points where our setups differ below. 

\paragraph{Data.} 
We use the publicly available Wiki40b dataset \cite{Guo_Dai_Vrandečić_Al-Rfou_2020}, a large text corpus derived from Wikipedia articles. 
We use only the English portion of this dataset because the language models that we consider were trained solely on English data. We randomly sample a subset of the data, of size $\approx 20$M tokens. 
We do not perform any pre-processing of the text, beyond that carried out by the native HuggingFace tokenisers for the respective language models. 
Unigram frequencies---which are used to estimate the unigram surprisals required by the Zipfian hypothesis---are computed on a separate subset of this same dataset.

\paragraph{Evaluation.}
We look at correlations between word lengths and the quantities put forward by various hypotheses about the influencing factors in a lexicon's word lengths. 
We expect to see that the hypotheses offering more accurate accounts of such factors have higher correlations with word lengths. In line with prior work \citep{piantadosi2011word,meylan2021challenges,levshina2022frequency,pimentel-etal-2023-revisiting}, we look at Spearman and Pearson correlations.

\clearpage
\section{Detailed Surprisal Theory Results}

\begin{table*}[h]
    \centering
    \resizebox{\textwidth}{!}{%
    \begin{tabular}{lccccccccccccc}
         \toprule
         &\multicolumn{3}{c}{Brown}
         &\multicolumn{3}{c}{Natural Stories}
         &\multicolumn{3}{c}{Provo}
         &\multicolumn{3}{c}{Dundee} \\
         \cmidrule(lr){2-4}
         \cmidrule(lr){5-7}
         \cmidrule(lr){8-10}
         \cmidrule(lr){11-13}
         Model
         & Improvement & Corrected & Buggy
         & Improvement & Corrected & Buggy
         & Improvement & Corrected & Buggy
         & Improvement & Corrected & Buggy \\
         \midrule
gpt2-small
& \phantom{-}{0.02}\phantom{$^{***}$}& \phantom{-}\textcolor{mygreen}{5.25}$^{***}$& \phantom{-}\textcolor{mygreen}{5.24}$^{***}$& \phantom{-}{0.02}\phantom{$^{***}$}& \phantom{-}\textcolor{mygreen}{4.35}$^{***}$& \phantom{-}\textcolor{mygreen}{4.33}$^{***}$& \phantom{-}{0.16}\phantom{$^{***}$}& \phantom{-}\textcolor{mygreen}{3.63}$^{***}$& \phantom{-}\textcolor{mygreen}{3.47}$^{***}$& \phantom{-}{0.01}\phantom{$^{***}$}& \phantom{-}\textcolor{mygreen}{1.07}$^{***}$& \phantom{-}\textcolor{mygreen}{1.06}$^{***}$ \\
gpt2-medium
& \phantom{-}{0.03}\phantom{$^{***}$}& \phantom{-}\textcolor{mygreen}{4.51}$^{***}$& \phantom{-}\textcolor{mygreen}{4.48}$^{***}$& \phantom{-}{0.02}\phantom{$^{***}$}& \phantom{-}\textcolor{mygreen}{4.08}$^{***}$& \phantom{-}\textcolor{mygreen}{4.07}$^{***}$& \phantom{-}{0.20}\phantom{$^{***}$}& \phantom{-}\textcolor{mygreen}{3.47}$^{***}$& \phantom{-}\textcolor{mygreen}{3.27}$^{***}$& \phantom{-}{0.01}\phantom{$^{***}$}& \phantom{-}\textcolor{mygreen}{1.01}$^{***}$& \phantom{-}\textcolor{mygreen}{1.00}$^{***}$ \\
gpt2-large
& \phantom{-}{0.04}\phantom{$^{***}$}& \phantom{-}\textcolor{mygreen}{4.53}$^{***}$& \phantom{-}\textcolor{mygreen}{4.49}$^{***}$& \phantom{-}{0.02}\phantom{$^{***}$}& \phantom{-}\textcolor{mygreen}{3.68}$^{***}$& \phantom{-}\textcolor{mygreen}{3.65}$^{***}$& \phantom{-}{0.21}\phantom{$^{***}$}& \phantom{-}\textcolor{mygreen}{3.10}$^{***}$& \phantom{-}\textcolor{mygreen}{2.89}$^{***}$& \phantom{-}{0.01}\phantom{$^{***}$}& \phantom{-}\textcolor{mygreen}{0.98}$^{***}$& \phantom{-}\textcolor{mygreen}{0.97}$^{***}$ \\
gpt2-xl
& \phantom{-}{0.03}\phantom{$^{***}$}& \phantom{-}\textcolor{mygreen}{4.23}$^{***}$& \phantom{-}\textcolor{mygreen}{4.20}$^{***}$& \phantom{-}{0.03}\phantom{$^{***}$}& \phantom{-}\textcolor{mygreen}{3.28}$^{***}$& \phantom{-}\textcolor{mygreen}{3.25}$^{***}$& \phantom{-}{0.19}\phantom{$^{***}$}& \phantom{-}\textcolor{mygreen}{3.09}$^{**}$\phantom{$^{*}$}& \phantom{-}\textcolor{mygreen}{2.90}$^{**}$\phantom{$^{*}$}& \phantom{-}{0.01}\phantom{$^{***}$}& \phantom{-}\textcolor{mygreen}{0.89}$^{***}$& \phantom{-}\textcolor{mygreen}{0.88}$^{***}$ \\
\midrule
pythia-70m
& \phantom{-}{0.01}\phantom{$^{***}$}& \phantom{-}\textcolor{mygreen}{4.70}$^{***}$& \phantom{-}\textcolor{mygreen}{4.69}$^{***}$& \phantom{-}{0.07}\phantom{$^{***}$}& \phantom{-}\textcolor{mygreen}{4.86}$^{***}$& \phantom{-}\textcolor{mygreen}{4.79}$^{***}$& \phantom{-}{0.10}\phantom{$^{***}$}& \phantom{-}\textcolor{mygreen}{3.69}$^{***}$& \phantom{-}\textcolor{mygreen}{3.59}$^{***}$& \phantom{-}{0.01}\phantom{$^{***}$}& \phantom{-}\textcolor{mygreen}{1.19}$^{***}$& \phantom{-}\textcolor{mygreen}{1.18}$^{***}$ \\
pythia-160m
& \phantom{-}{0.02}\phantom{$^{***}$}& \phantom{-}\textcolor{mygreen}{4.81}$^{***}$& \phantom{-}\textcolor{mygreen}{4.78}$^{***}$& \phantom{-}\textcolor{mygreen}{0.05}$^{*}$\phantom{$^{**}$}& \phantom{-}\textcolor{mygreen}{4.27}$^{***}$& \phantom{-}\textcolor{mygreen}{4.22}$^{***}$& \phantom{-}{0.15}\phantom{$^{***}$}& \phantom{-}\textcolor{mygreen}{3.61}$^{***}$& \phantom{-}\textcolor{mygreen}{3.46}$^{***}$& \phantom{-}{0.01}\phantom{$^{***}$}& \phantom{-}\textcolor{mygreen}{1.14}$^{***}$& \phantom{-}\textcolor{mygreen}{1.13}$^{***}$ \\
pythia-410m
& \phantom{-}{0.03}\phantom{$^{***}$}& \phantom{-}\textcolor{mygreen}{4.34}$^{***}$& \phantom{-}\textcolor{mygreen}{4.31}$^{***}$& \phantom{-}\textcolor{mygreen}{0.05}$^{**}$\phantom{$^{*}$}& \phantom{-}\textcolor{mygreen}{3.86}$^{***}$& \phantom{-}\textcolor{mygreen}{3.81}$^{***}$& \phantom{-}{0.20}\phantom{$^{***}$}& \phantom{-}\textcolor{mygreen}{3.24}$^{***}$& \phantom{-}\textcolor{mygreen}{3.04}$^{***}$& \phantom{-}{0.01}\phantom{$^{***}$}& \phantom{-}\textcolor{mygreen}{1.05}$^{***}$& \phantom{-}\textcolor{mygreen}{1.04}$^{***}$ \\
pythia-14b
& \phantom{-}{0.03}\phantom{$^{***}$}& \phantom{-}\textcolor{mygreen}{4.01}$^{***}$& \phantom{-}\textcolor{mygreen}{3.98}$^{***}$& \phantom{-}\textcolor{mygreen}{0.04}$^{*}$\phantom{$^{**}$}& \phantom{-}\textcolor{mygreen}{3.24}$^{***}$& \phantom{-}\textcolor{mygreen}{3.20}$^{***}$& \phantom{-}{0.16}\phantom{$^{***}$}& \phantom{-}\textcolor{mygreen}{2.80}$^{***}$& \phantom{-}\textcolor{mygreen}{2.64}$^{***}$& \phantom{-}{0.01}\phantom{$^{***}$}& \phantom{-}\textcolor{mygreen}{0.96}$^{***}$& \phantom{-}\textcolor{mygreen}{0.95}$^{***}$ \\
pythia-28b
& \phantom{-}{0.04}\phantom{$^{***}$}& \phantom{-}\textcolor{mygreen}{3.92}$^{***}$& \phantom{-}\textcolor{mygreen}{3.89}$^{***}$& \phantom{-}{0.03}\phantom{$^{***}$}& \phantom{-}\textcolor{mygreen}{2.96}$^{***}$& \phantom{-}\textcolor{mygreen}{2.94}$^{***}$& \phantom{-}{0.18}\phantom{$^{***}$}& \phantom{-}\textcolor{mygreen}{3.09}$^{***}$& \phantom{-}\textcolor{mygreen}{2.91}$^{***}$& \phantom{-}{0.01}\phantom{$^{***}$}& \phantom{-}\textcolor{mygreen}{0.88}$^{***}$& \phantom{-}\textcolor{mygreen}{0.87}$^{***}$ \\
pythia-69b
& \phantom{-}{0.04}\phantom{$^{***}$}& \phantom{-}\textcolor{mygreen}{3.59}$^{***}$& \phantom{-}\textcolor{mygreen}{3.55}$^{***}$& \phantom{-}{0.03}\phantom{$^{***}$}& \phantom{-}\textcolor{mygreen}{2.55}$^{***}$& \phantom{-}\textcolor{mygreen}{2.52}$^{***}$& \phantom{-}{0.16}\phantom{$^{***}$}& \phantom{-}\textcolor{mygreen}{2.61}$^{***}$& \phantom{-}\textcolor{mygreen}{2.46}$^{***}$& \phantom{-}{0.01}\phantom{$^{***}$}& \phantom{-}\textcolor{mygreen}{0.81}$^{***}$& \phantom{-}\textcolor{mygreen}{0.80}$^{***}$ \\
pythia-120b
& \phantom{-}{0.04}\phantom{$^{***}$}& \phantom{-}\textcolor{mygreen}{3.51}$^{***}$& \phantom{-}\textcolor{mygreen}{3.46}$^{***}$& \phantom{-}{0.03}\phantom{$^{***}$}& \phantom{-}\textcolor{mygreen}{2.47}$^{***}$& \phantom{-}\textcolor{mygreen}{2.45}$^{***}$& \phantom{-}{0.17}\phantom{$^{***}$}& \phantom{-}\textcolor{mygreen}{2.36}$^{***}$& \phantom{-}\textcolor{mygreen}{2.19}$^{***}$& \phantom{-}{0.01}\phantom{$^{***}$}& \phantom{-}\textcolor{mygreen}{0.76}$^{***}$& \phantom{-}\textcolor{mygreen}{0.75}$^{***}$ \\
        \bottomrule
    \end{tabular}%
    }
    \caption{$\deltallh$ between regressors with and without surprisal as a predictor.}
    \label{tab:surp_theory_full_table}
\end{table*}

\section{Multilingual Surprisal Theory Results} \label{app:multilingual_surprisal}

In this section, we expand our surprisal theory experiments (in \cref{sec:surprisal_theory}) to multiple languages, following a similar experimental setup to \citet{wilcox2023testing}.
Specifically, we analyse the MECO dataset \citep{siegelman2022expanding}, running our experiments on seven of its languages: Finnish, German, Greek, Hebrew, Italian, Spanish, and Turkish.
We estimate surprisals for the words in these languages using mGPT \citep{shliazhko2022mgpt}---a language model defined over the output of a \bow-marking tokeniser.
We thus analyse the effect of our correction \usebugfixcounter{bugfix:bow} when using this model to estimate surprisals.
The $\deltallh$ when predicting reading times on this dataset are presented in \cref{fig:multilingual_llh}.\looseness=-1

\begin{figure}[h]
    \centering
    \includegraphics[width=.8\linewidth]{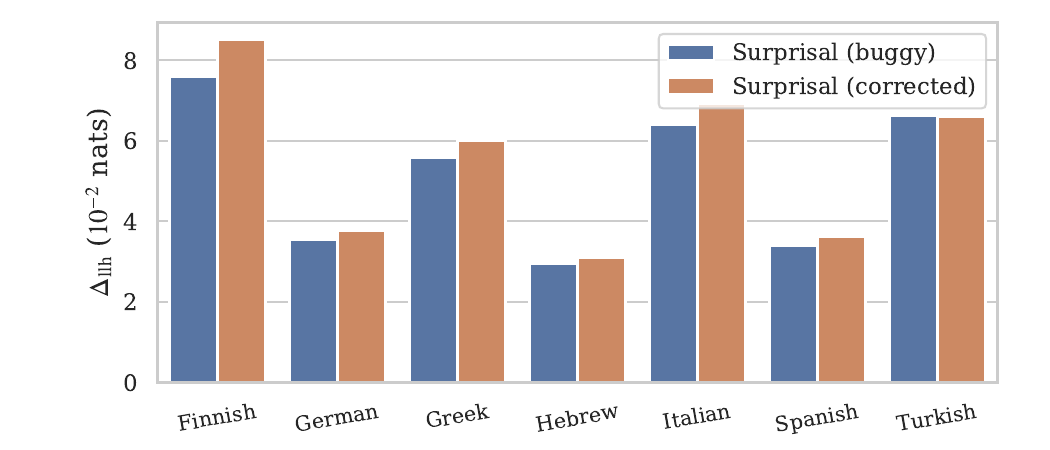}
    \caption{$\deltallh$ between regressors with and without surprisal as a predictor in a subset of the languages in the MECO dataset \citep{siegelman2022expanding}.}
    \label{fig:multilingual_llh}
\end{figure}

\section{Proofs of Lemmas and Theorems}

\subsection{Proof of End-of-Word Tokeniser's \texorpdfstring{\Cref{thm:probofeowtokeniser}}{Theorem}}
\label{app:proof_theorem_probofeowtokeniser}

\begin{restatable}{lemma}{setequivofeowtokeniser} \label{lemma:prefix_prob_equivs_eow_tokeniser}
    Let $\wordsubwordsmap$ be a \eow-marking tokeniser.
    We can show the following equivalence:
    \begin{align}
        &\pprefixwords(\words_{<t} \circ \lexicon^*) = \pprefixsubwords(\subwordschosenprefix \circ \vocab^*) \\
        &\pprefixwords(\words_{<t} \circ \word \circ \lexicon^*) = \pprefixsubwords(\subwordschosenprefix \circ \subwordschosenword \circ \vocab^*) \nonumber
    \end{align}
\end{restatable}
\begin{proof}
    This lemma assumes a segmentation-compatible tokeniser.
    Therefore, we can rely on \cref{defn:pretokenised_tokeniser}, whose equation we rewrite here for convenience:
    \begin{align}
        \wordsubwordsmapmath(\words) = \wordsubwordmapmath(\word_1) \circ \wordsubwordmapmath(\word_2) \circ \cdots \circ  \wordsubwordmapmath(\word_{|\words|})
    \end{align}
    Further, as this tokeniser is \eow-marking, we have that: $\wordsubwordmap: \lexicon \to \vocabmid^{*} \circ \vocabeow$.
    We now prove the equivalences above.
    First, we show that $\words' \in (\words_{<t} \circ \lexicon^*) \implies \wordsubwordsmap(\words') \in (\subwordschosenprefix \circ \vocab^*)$; 
    this shows that the tokenised version of all strings $\words' \in (\words_{<t} \circ \lexicon^*)$ are present in the set $(\subwordschosenprefix \circ \vocab^*)$.
    \begin{subequations}
    \begin{align}
        \words_{<t} \circ \lexicon^* 
        &= \left\{\words_{<t} \circ \words' \mid \words' \in \lexicon^* \right\} & \mathcomment{definition of $\circ$} \\
        &\stringequiv \left\{\wordsubwordsmapmath(\words_{<t} \circ \words') \mid \words' \in \lexicon^* \right\} & \mathcomment{definition of $\stringequiv$} \\
        &= \left\{\wordsubwordsmapmath(\words_{<t}) \circ \wordsubwordsmapmath(\words') \mid \words' \in \lexicon^* \right\} & \mathcomment{decomposition of $\wordsubwordsmapmath$} \\
        &= \wordsubwordsmapmath(\words_{<t}) \circ \left\{\wordsubwordsmapmath(\words') \mid \words' \in \lexicon^* \right\} & \mathcomment{definition of $\circ$ over sets} \\
        &= \subwordschosenprefix \circ \left\{\wordsubwordsmapmath(\words') \mid \words' \in \lexicon^* \right\}  & \mathcomment{definition of $\subwordschosenprefix$} \\
        &\subseteq \subwordschosenprefix \circ \vocab^*
    \end{align}
    \end{subequations}
    \newcommand{\wordsprefixset}{\subwordsset^{\words_{<t} \circ \lexicon^*}}
    We now define the set $\wordsprefixset \defeq \left\{\wordsubwordsmap(\words') \mid \words' \in (\words_{<t} \circ \lexicon^*) \right\}$, and note that $\words_{<t} \circ \lexicon^*  \stringequiv \wordsprefixset$. 
    We can thus split the probability we are computing into two parts:
    \begin{align}
        \pprefixsubwords(\subwordschosenprefix \circ \vocab^*)
        &= \pprefixsubwords(\wordsprefixset) + 
        \pprefixsubwords((\subwordschosenprefix \circ \vocab^*) \setminus \wordsprefixset) 
    \end{align}
    If we prove that $\pprefixsubwords((\subwordschosenprefix \circ \vocab^*) \setminus \wordsprefixset) = 0$, then it must be that $\pprefixwords(\words_{<t} \circ \lexicon^*) = \pprefixsubwords(\subwordschosenprefix \circ \vocab^*)$, which completes our proof. Towards this end, we show that $\subwords' \in (\subwordschosenprefix \circ \vocab^*) \implies \subwordwordsmap(\subwords') \in (\words_{<t} \circ \lexicon^*)$. For the reader's convenience, we first rewrite \cref{eq:prob_subwords_from_words_with_zero} here:
    \begin{align}\label{eq:prob_subwords_from_words_with_zero2}
        p(\subwords) = \left\{ 
        \begin{array}{cr}
             p(\words) & \texttt{if } \subwords = \wordsubwordsmapmath(\words) \\
             0 & \texttt{if } \subwords \in \unmmapedsubwords
        \end{array}
        \right.
    \end{align}
    We now proceed with our proof. 
    \begin{subequations}
    \begin{align}
        \subwordschosenprefix \circ \vocab^*
        &= \left\{\subwordschosenprefix \circ \subwords' \mid \subwords' \in \vocab^* \right\} & \mathcomment{definition of $\circ$} \\
        &\!\stringequivtowords \left\{\subwordwordsmapmath(\subwordschosenprefix \circ \subwords') \mid \subwords' \in \vocab^* \right\} & \mathcomment{definition of $\stringequivtowords$} \\
        &= \left\{\subwordwordsmapmath(\subwordschosenprefix) \circ \subwordwordsmapmath(\subwords') \mid \subwords' \in \vocab^* \right\} & \mathcomment{$\subwordschosenprefix$ ends in $\vocabeow$, decomposition of $\subwordwordsmapmath$} \\
        &= \subwordwordsmapmath(\subwordschosenprefix) \circ \left\{\subwordwordsmapmath(\subwords') \mid \subwords' \in \vocab^* \right\} & \mathcomment{definition of $\circ$} \\
        &= \words_{<t} \circ \left\{\subwordwordsmapmath(\subwords') \mid \subwords' \in \vocab^* \right\} & \mathcomment{definition of $\subwordschosenprefix$} \\
        &= \words_{<t} \circ \lexicon^* & \mathcomment{co-domain of $\subwordwordsmapmath$}
    \end{align}
    \end{subequations}
    This result implies that any string $\subwords \in (\subwordschosenprefix \circ \vocab^*)$ is either mapped to from a string $\words' \in (\words_{<t} \circ \lexicon^*)$, or not mapped to at all by the tokenisation function $\wordsubwordsmap$.
    (We note again that $\wordsubwordsmap$ only maps each $\words$ to a single subword sequence $\subwords$, even if multiple subword sequences would be detokenised to the same $\words$.)
    As $\wordsprefixset$ is defined as the set of all subword sequences mapped to from $\words' \in (\words_{<t} \circ \lexicon^*)$, we have that
    $((\subwordschosenprefix \circ \vocab^*) \setminus \wordsprefixset) \subseteq \unmmapedsubwords$.
    It follows that the probability of the set $(\subwordschosenprefix \circ \vocab^*)$ includes the probability of no other string $\words' \notin (\words_{<t} \circ \lexicon^*)$.
    By the property $((\subwordschosenprefix \circ \vocab^*) \setminus \wordsprefixset) \subseteq \unmmapedsubwords$, we have that 
    $\pprefixsubwords((\subwordschosenprefix \circ \vocab^*) \setminus \wordsprefixset) = 0$, which completes the proof.
\end{proof}

\probofeowtokeniser*
\begin{proof}
    The first part of this theorem utilizes \cref{lemma:prefix_prob_equivs_eow_tokeniser}.
    We can then derive the probabilities in \cref{eq:prob_eow_tokeniser} as:
    \begin{subequations}
    \begin{align}
        p(\word \mid \words_{<t}) 
        &= \frac{\pprefix(\words_{< t} \circ \word \circ \lexicon^*)}{\pprefix(\words_{<t} \circ \lexicon^*)} \\
        &= \frac{\pprefix\left(\subwordschosenfull \circ \vocab^*\right)}{\pprefix\left(\subwordschosenprefix \circ \vocab^*\right)} \\
        &= \frac{\prod\limits_{t'=1}^{\left|\subwordschosenfull\right|} p\left(\subwordschosenfull_{t'} \mid \subwordschosenfull_{<t'}\right)}
        {\prod\limits_{t'=1}^{\left|\subwordschosenprefix\right|} p\left(\subwordschosenprefix_{t'} \mid \subwordschosenprefix_{<t'}\right)} \\
        &= \prod\limits_{t'=\left|\subwordschosenprefix\right|+1}^{\left|\subwordschosenfull\right|} p\left(\subwordschosenfull_{t'} \mid \subwordschosenfull_{<t'}\right) \\
        &= \prod\limits_{t'=1}^{\left|\subwordschosenword\right|} p\left(\subwordschosenword_{t'} \mid \subwordschosenprefix \circ \subwordschosenword_{<t'}\right)
    \end{align}
    \end{subequations}
    This completes the proof.
\end{proof}

\subsection{Proof of Beginning-of-Word Tokeniser's \texorpdfstring{\Cref{thm:probofbowtokeniser}}{Theorem}}
\label{app:proof_theorem_probofbowtokeniser}

\begin{restatable}{lemma}{setequivofbowtokeniser} \label{lemma:prefix_prob_equivs_bow_tokeniser}
    Let $\wordsubwordsmap$ be a \bow-marking tokeniser.
    We can show the following equivalence:
    \begin{align}
        &\pprefixwords(\words_{<t} \circ \lexicon^*) = \pprefixsubwords(\subwordschosenprefix \circ \vocabboweospluscont) \\
        &\pprefixwords(\words_{<t} \circ \word \circ \lexicon^*) = \pprefixsubwords(\subwordschosenprefix \circ \subwordschosenword \circ \vocabboweospluscont) \nonumber
    \end{align}
\end{restatable}
\begin{proof}
    This lemma assumes a segmentation-compatible tokeniser.
    Therefore, we can rely on \cref{defn:pretokenised_tokeniser}, whose mathematical formulation we rewrite here for convenience:
    \begin{align}
        \wordsubwordsmapmath(\words) = \wordsubwordmapmath(\word_1) \circ \wordsubwordmapmath(\word_2) \circ \cdots \circ  \wordsubwordmapmath(\word_{|\words|})
    \end{align}
    Further, as this tokeniser is \bow-marking, we have that: $\wordsubwordmap: \lexicon \to \vocabbow \circ \vocabmid^{*}$.
    We now prove the equivalences above.
    First, we show that $\words' \in (\words_{<t} \circ \lexicon^*) \implies \wordsubwordsmap(\words') \in (\subwordschosenprefix \circ \vocabboweospluscont)$; 
    this shows that the tokenised version of all strings $\words' \in (\words_{<t} \circ \lexicon^*)$ are present in the set $(\subwordschosenprefix \circ \vocabboweospluscont)$.%
    \begin{subequations}
    \begin{align}
        \words_{<t} \circ \lexicon^* 
        &= \left\{\words_{<t} \circ \words' \mid \words' \in \lexicon^* \right\} & \mathcomment{definition of $\circ$} \\
        &\stringequiv \left\{\wordsubwordsmapmath(\words_{<t} \circ \words') \mid \words' \in \lexicon^* \right\} & \mathcomment{definition of $\stringequivtosubwords$} \\
        &= \left\{\wordsubwordsmapmath(\words_{<t}) \circ \wordsubwordsmapmath(\words') \mid \words' \in \lexicon^* \right\} & \mathcomment{decomposition of $\wordsubwordsmapmath$} \\
        &= \wordsubwordsmapmath(\words_{<t}) \circ \left\{\wordsubwordsmapmath(\words') \mid \words' \in \lexicon^* \right\} & \mathcomment{definition of $\circ$ over sets} \\
        &= \subwordschosenprefix \circ \left\{\wordsubwordsmapmath(\words') \mid \words' \in \lexicon^* \right\}  & \mathcomment{definition of $\subwordschosenprefix$} \\
        &= \subwordschosenprefix \circ \left(\{\eos \} \cup \left(\vocabbow \circ \vocabmid^* \circ \left\{\wordsubwordsmapmath(\words') \mid \words' \in \lexicon^* \right\}\right) \right)
        \!\!\!\!\!\!\!\!\!\!\!\! \\
        &\subseteq \subwordschosenprefix \circ \left(\{\eos \} \cup \left(\vocabbow \circ \vocabmid^* \circ \vocab^* \right) \right) \\
        &\subseteq \subwordschosenprefix \circ \vocabboweospluscont
    \end{align}
    \end{subequations}
    \newcommand{\wordsprefixset}{\subwordsset^{\words_{<t} \circ \lexicon^*}}
    We now define the set $\wordsprefixset \defeq \left\{\wordsubwordsmap(\words') \mid \words' \in (\words_{<t} \circ \lexicon^*) \right\}$, and note that $\words_{<t} \circ \lexicon^*  \stringequiv \wordsprefixset$. 
    We can thus split the probability we are computing into two parts:
    \begin{align}
        \pprefixsubwords(\subwordschosenprefix \circ \vocabboweospluscont)
        &= \pprefixsubwords(\wordsprefixset) + 
        \pprefixsubwords((\subwordschosenprefix \circ \vocabboweospluscont) \setminus \wordsprefixset) 
    \end{align}
    By the same logic as in \cref{lemma:prefix_prob_equivs_eow_tokeniser}, if we prove that $\pprefixsubwords((\subwordschosenprefix \circ \vocabboweospluscont) \setminus \wordsprefixset) = 0$, then we have that $\pprefixwords(\words_{<t} \circ \lexicon^*) = \pprefixsubwords(\subwordschosenprefix \circ \vocabboweospluscont)$.
    To this end, we show that $\subwords' \in (\subwordschosenprefix \circ \vocabboweospluscont) \implies \subwordwordsmap(\subwords') \in (\words_{<t} \circ \lexicon^*)$. 
    As with \cref{lemma:prefix_prob_equivs_eow_tokeniser}, this result implies that the tokenised version of no other strings $\words' \notin (\words_{<t} \circ \lexicon^*)$ are present in the set $(\subwordschosenprefix \circ \vocabboweospluscont)$, which itself implies that $(\subwordschosenprefix \circ \vocabboweospluscont) \setminus \wordsprefixset \subseteq \unmmapedsubwords$.
    \begin{subequations}
    \begin{align}
        \subwordschosenprefix \circ \vocabboweospluscont
        &= \left\{\subwordschosenprefix \circ \subwords' \mid \subwords' \in \vocabboweospluscont \right\} & \mathcomment{definition of $\circ$} \\
        &\stringequivtowords \left\{\subwordwordsmapmath(\subwordschosenprefix \circ \subwords') \mid \subwords' \in \vocabboweospluscont \right\} & \mathcomment{definition of $\stringequivtowords$} \\
        &= \left\{\wordsubwordsmapmath(\subwordschosenprefix) \circ \wordsubwordsmapmath(\subwords') \mid \subwords' \in \vocabboweospluscont \right\} \\
        && 
        \!\!\!\!\!\!\!\!\!\!\!\!\!\!\!\!\!\!\!\!\!\!\!\!\!
        \!\!\!\!\!\!\!\!\!\!\!\!\!\!\!\!\!\!\!\!\!\!\!\!\!
        \!\!\!\!\!\!\!\!\!\!\!\!\!\!\!\!\!\!\!\!\!\!\!\!\!
        \!\!\!\!\!\!\!\!\!\!\!\!\!\!\!\!\!\!\!\!\!\!\!\!\!
        \mathcomment{$\subwords'$ is either empty, or starts in $\vocabbow$, $\wordsubwordsmapmath$ thus decomposes} \\
        &= \wordsubwordsmapmath(\subwordschosenprefix) \circ \left\{\wordsubwordsmapmath(\subwords') \mid \subwords' \in \vocabboweospluscont \right\} & \qquad \mathcomment{definition of $\circ$ over sets} \\
        &= \words_{<t} \circ \left\{\wordsubwordsmapmath(\subwords') \mid \subwords' \in \vocabboweospluscont \right\} & \mathcomment{definition of $\subwordschosenprefix$} \\
        &= \words_{<t} \circ \lexicon^* & \mathcomment{co-domain of $\wordsubwordsmapmath$}
    \end{align}
    \end{subequations}
    Since $\subwords \in ((\subwordschosenprefix \circ \vocabboweospluscont) \setminus \wordsprefixset) \implies \subwords \in \unmmapedsubwords$, we have that 
    $\pprefixsubwords((\subwordschosenprefix \circ \vocabboweospluscont) \setminus \wordsprefixset) \!=\! 0$, which completes the proof.
\end{proof}

\probofbowtokeniser*
\begin{proof}
    The first part of this theorem simply re-writes \cref{lemma:prefix_prob_equivs_bow_tokeniser}.
    We now derive the probabilities in \cref{eq:prob_bow_tokeniser} as:
\begin{subequations}
\begin{align}
    p(\word \mid \words_{<t}) 
    &= \frac{\pprefix(\words_{< t} \circ \word \circ \lexicon^*)}{\pprefix(\words_{<t} \circ \lexicon^*)} \\
    &= \frac{\sum_{\{\subword \in \vocabboweos\}} \pprefix\left(\subwordschosenfull \circ \subword \circ \vocab^*\right)}{\sum_{\{\subword \in \vocabboweos\}} \pprefix\left(\subwordschosenprefix \circ \subword \circ \vocab^*\right)} \\
    &= \frac{\sum_{\{\subword \in \vocabboweos\}} p\left(\subword \mid \subwordschosenfull\right)\,\prod\limits_{t'=1}^{\left|\subwordschosenfull\right|} p\left(\subwordschosenprefix_{t'} \mid \subwordschosenprefix_{<t'}\right)}
    {\sum_{\{\subword \in \vocabboweos\}} p\left(\subword \mid \subwordschosenprefix\right)\,\prod\limits_{t'=1}^{\left|\subwordschosenprefix\right|} p\left(\subwordschosenprefix_{t'} \mid \subwordschosenprefix_{<t'}\right)} \\
    &= \frac{\prod\limits_{t'=1}^{\left|\subwordschosenfull\right|} p\left(\subwordschosenfull_{t'} \mid \subwordschosenfull_{<t'}\right)\,\sum_{\{\subword \in \vocabboweos\}} p\left(\subword \mid \subwordschosenfull\right)}
    {\prod\limits_{t'=1}^{\left|\subwordschosenprefix\right|} p\left(\subwordschosenprefix_{t'} \mid \subwordschosenprefix_{<t'}\right)\,\sum_{\{\subword \in \vocabboweos\}} p\left(\subword \mid \subwordschosenprefix\right)} \\
    &= \frac{\prod\limits_{t'=\left|\subwordschosenprefix+1\right|}^{\left|\subwordschosenfull\right|} p\left(\subwordschosenfull_{t'} \mid \subwordschosenfull_{<t'}\right)\,\sum_{\{\subword \in \vocabboweos\}} p\left(\subword \mid \subwordschosenfull\right)}
    {\sum_{\{\subword \in \vocabboweos\}} p\left(\subword \mid \subwordschosenprefix\right)} \\
    &= \prod\limits_{t'=1}^{\left|\subwordschosenword\right|} p\left(\subwordschosenword_{t'} \mid \subwordschosenprefix \circ \subwordschosenword_{<t'}\right)\,
    \frac{\sum_{\{\subword \in \vocabboweos\}} p\left(\subword \mid \subwordschosenfull\right)}
    {\sum_{\{\subword \in \vocabboweos\}} p\left(\subword \mid \subwordschosenprefix\right)}
\end{align}
\end{subequations}
This completes the proof.
\end{proof}

\subsection{Theorem of Non-\eow-marking Final-word Tokeniser's}
\label{app:proof_theorem_probofeowateostokeniser}

\begin{restatable}{theorem}{probofeowateostokeniser} \label{thm:probofeowateostokeniser}
    Let $\wordsubwordsmap$ be a \eow-marking tokeniser with unmarked final word.
    We can show the following equivalence:
    \begin{align}\label{eq:eowateos_prob_equivalence}
        &\pprefixwords(\words_{<t} \circ \lexicon^*) = \pprefixsubwords(\subwordschosenprefix \circ \vocab^*) \\
        &\pprefixwords(\words_{<t} \circ \word \circ \lexicon^*) = 
        \pprefixsubwords((\subwordschosenprefix \circ \subwordschosenword \circ \vocab^*) \cup \{\subwordschosenprefix \circ \subwordschosenword_{\midstring}\}) \nonumber
    \end{align}
    Further, we can compute a word's probability as:
    \begin{align} \label{eq:prob_eoweos_tokeniser_theorem}
        &p(\word \mid \words_{<t}) = 
        \bigg(\!
        p\left(\subwordschosenword_{\midstring} \!\mathop{\mid} \subwordschosenprefix\right) \!\! \underbrace{\sum_{\subword \in \vocabpunct} \!\! p(\subword\!\mathop{\mid} \subwordschosenprefix\!\mathop{\circ} \subwordschosenword_{\midstring})
        \!\bigg) \!+\! 
        p(\subwordschosenword\!\mathop{\mid} \subwordschosenprefix)
        }_{\text{\usebugfixcounter{bugfix:eowateos}}}
    \end{align}
\end{restatable}

\subsection{Theorem of Non-\bow-marking First-word Tokeniser's}
\label{app:proof_theorem_probofbowatbostokeniser}

\begin{restatable}{theorem}{probofbowatbostokeniser} \label{thm:probofbowatbostokeniser}
    Let $\wordsubwordsmap$ be a \bow-marking tokeniser with unmarked first words.
    We can show the following equivalence:
    \begin{align}\label{eq:bow_at_bos_prob_equivalence}
        &\pprefixwords(\lexicon^*) = \pprefixsubwords(\vocabmideospluscont) \\
        &\pprefixwords(\word \circ \lexicon^*) = 
        \pprefixsubwords(\subwordschosenword_{\midstring} \circ \vocabboweospluscont) \nonumber
    \end{align}
    Further, we can compute a word's probability as:
    \begin{align} \label{eq:prob_bowatbos_tokeniser_theorem}
        &p(\word \mid \words_{<t}) = 
        p\left(\subwordschosenword_{\midstring} \mid \emptystring \right)\,
        \underbrace{\frac{\sum_{\{\subword \in \vocabboweos \}} p\left(\subword \mid \subwordschosenword\right)}
        {\sum_{\{\subword \in \vocabmideos\}} p\left(\subword \mid \emptystring \right)}}_{\text{\usebugfixcounter{bugfix:bowatbos}}}
    \end{align}
\end{restatable}

\end{document}